\documentclass{article}

\usepackage[preprint]{neurips_2024}

\usepackage{amsthm}
\usepackage{graphicx}
\RequirePackage{algorithm}
\RequirePackage{algorithmic}
\usepackage{makecell}

\usepackage[svgnames]{xcolor}
\usepackage{framed}
\definecolor{light-gray}{gray}{0.95}
\definecolor{shadecolor}{named}{light-gray}

\newcommand{\algo}{DRO}

\usepackage{enumitem}

\usepackage[utf8]{inputenc} 
\usepackage[T1]{fontenc}    
\usepackage{hyperref}       
\usepackage{url}            
\usepackage{booktabs}       
\usepackage{amsfonts}       
\usepackage{nicefrac}       
\usepackage{microtype}      
\usepackage{xcolor}         
\usepackage{amsmath, amsfonts, bm, dsfont}

\newtheorem{proposition}{Proposition}

\newtheorem*{proposition*}{Proposition}

\def\eqref#1{Equation~(\ref{#1})}

\def\1{\bm{1}}

\DeclareMathAlphabet{\mathsfit}{\encodingdefault}{\sfdefault}{m}{sl}
\SetMathAlphabet{\mathsfit}{bold}{\encodingdefault}{\sfdefault}{bx}{n}

\newcommand{\E}{\mathbb{E}}

\newcommand{\KL}{D_{\mathrm{KL}}}

\renewcommand{\hat}{\widehat}
\renewcommand{\phi}{\varphi}

\newcommand{\B}{\mathcal{B}}

\renewcommand{\L}{\mathcal{L}}

\newcommand{\beq}{\begin{equation}}
\newcommand{\eeq}{\end{equation}}

\newcommand{\beqa}{\begin{eqnarray}}
\newcommand{\eeqa}{\end{eqnarray}}

\newcommand{\beqan}{\begin{eqnarray*}}
\newcommand{\eeqan}{\end{eqnarray*}}

\renewcommand{\E}{\mathbb{E}}

\renewcommand{\1}[1]{\mathbb{1}\{#1\}}

\newcommand{\piref}{\pi_{\texttt{ref}}}

\renewcommand{\KL}{\mbox{KL}}

\newcommand{\eqdef}{\stackrel{\rm def}{=}}

\newcommand{\supp}{\mathop{\mathrm{supp}}}
\newtheorem{theorem}{Theorem}

\usepackage{xcolor}

\hypersetup{
    colorlinks=true,
    linkcolor=blue,
    citecolor=cyan,
}

\title{Offline Regularised Reinforcement Learning for \\ Large Language Models Alignment}

\author{%
Pierre Harvey Richemond\thanks{Equal contribution. Correspondence at richemond@google.com or piot@google.com.} \\
Google DeepMind \\
\And
Yunhao Tang$^{*}$ \\
Google DeepMind \\
\And
Daniel Guo \\
Google DeepMind
\And
Daniele Calandriello \\
Google DeepMind
\And
Mohammad Gheshlaghi Azar \\
Cohere \\
\And
Rafael Rafailov \\
Stanford University \\
\And
Bernardo Avila Pires \\
Google DeepMind \\
\And
Eugene Tarassov \\
Google DeepMind \\
\And
Lucas Spangher \\
Google DeepMind \\
\And
Will Ellsworth \\
Google DeepMind \\
\And
Aliaksei Severyn \\
Google Research \\
\And
Jonathan Mallinson \\
Google Research \\
\And
Lior Shani \\
Google Research \\
\And
Gil Shamir \\
Google DeepMind \\
\And
Rishabh Joshi \\
Google DeepMind \\
\And
Tianqi Liu \\
Google DeepMind \\
\And
Remi Munos$^{*}$ \\
Google DeepMind \\
\And
Bilal Piot$^{*}$ \\
Google DeepMind \\
}

\begin{document}

\maketitle

\begin{abstract}
The dominant framework for alignment of large language models (LLM), whether through reinforcement learning from human feedback or direct preference optimisation, is to learn from preference data. This involves building datasets where each element is a quadruplet composed of a prompt, two independent responses (completions of the prompt) and a human preference between the two independent responses, yielding a preferred and a dis-preferred response. Such data is typically scarce and expensive to collect. On the other hand, \emph{single-trajectory} datasets where each element is a triplet composed of a prompt, a response and a human feedback is naturally more abundant. The canonical element of such datasets is for instance an LLM's response to a user's prompt followed by a user's feedback such as a thumbs-up/down. Consequently, in this work, we propose \algo, or \emph{Direct Reward Optimisation}, as a framework and associated algorithms that do not require pairwise preferences. \algo~uses a simple mean-squared objective that can be implemented in various ways. We validate our findings empirically, using T5 encoder-decoder language models, and show \algo's performance over selected baselines such as Kahneman-Tversky Optimization (KTO). Thus, we confirm that \algo~is a simple and empirically compelling method for single-trajectory policy optimisation.
\end{abstract}

\section{Introduction}

Aligning the behavior of artificial agents with human preferences is critical for improving quality, helpfulness and safety~\citep{bai2022training} of agents' responses. The most established methodology for human alignment is Reinforcement Learning from Human Feedback (RLHF)~\citep{knox2008tamer,griffith2013policy,christiano2017deep,warnell2018deep} which consists of fine-tuning pre-trained Large Language Models (LLMs)~\citep{glaese2022improving,chatgpt}. More precisely, it typically entails learning a reward model under the Bradley-Terry model~\citep{bradley1952rank} of human preferences and subsequently employing reinforcement learning (RL) to optimise the LLM's performance as judged by this reward model~\citep{christiano2017deep,ziegler2020finetuning}. This optimisation is done by generating a set of responses and their associated rewards from a set of chosen prompts. Therefore, on top of learning an additional reward model, this method requires sampling from the LLM at training time, which is costly and technically challenging.

\citet{rafailov2023direct} introduced an alternative, reward-free and sampling-free method termed direct preference optimisation (DPO). This popular approach uses a supervised objective that contrasts pairs of responses to a specific prompt. DPO is able to circumvent the explicit learning of a reward signal, while remaining mathematically equivalent to the traditional RL approach, as proven by~\citet{azar2023general}. Such offline preference optimisation method has been extended to a few variants in follow-up work (see, e.g., \citep{zhao2023slichf,tang2024generalized}) and gained popularity in practice.

Despite DPO's widespread use, there is still an important remaining shortcoming to this approach which is the high-cost of collecting human preferences. Establishing human preferences might over time become a self-defeating endeavour: as LLMs improve in quality, the task of distinguishing between a pair of strong responses gets increasingly difficult (see, e.g., arguments in \citep{saunders2022self,bowman2022measuring}), and would require additional efforts in improving both the quality and scale of the collected human feedback.

Second and more importantly, annotating pairwise data is more expensive and less natural than simply indicating whether a single completion is satisfactory or not, e.g., by assigning a binary \emph{thumbs up} or \emph{down} rating to the model completion. The former is generally carried out by paid raters whereas the latter could be produced by users at a much larger scale. Consequently, \emph{single-trajectory} data is much more abundant in the wild, hence, cheaper and more easily collected than scarce preference data. Leveraging single trajectory data promises to unlock the benefits associated with scale in deep learning. Analogous algorithmic advances motivated by the need to decrease supervision signals have often resulted in important step-changes in the empirical capacities of deep networks, e.g. through self-supervised systems \citep{devlin2019bert, chen2020simple, BYOL, radford2021learning}.

In order to exploit this single trajectory setting, we introduce Direct Reward Optimisation (\algo). \algo~is a framework, derived from mathematical first principles. \algo~is designed to work in the \emph{offline} single-trajectory setting with human feedback. Specifically, our contributions are as follows: 
\begin{itemize}[leftmargin=*]
    \item We introduce \algo~as a generic framework performing single-trajectory RLHF optimisation thanks to a simple quadratic objective and perform theoretical analysis.
    \item  We propose a practical instantiation of \algo, $\texttt{DRO-V}$ which combines offline policy learning with a value function learning, and hence the suffix $\texttt{-V}$.
    \item We compare \texttt{\algo-V}~against Kahneman-Tversky Optimization (KTO)~\citep{kto}, an algorithm that has also been specifically designed for the single-trajectory setting. We find that \texttt{DRO-V} significantly outperforms KTO, when using T5 encoders~\citep{JMLR:v21:20-074} with up to $3$ billion parameters, on the \emph{UltraFeedback} dataset~\citep{cui2023ultrafeedback}. We perform several ablations to investigate and understand our algorithm's empirical performance.
\end{itemize}

\section{Background}
\label{sec:background}

Here we introduce background for RLHF and a few important alignment algorithms.

\textbf{Standard pairwise alignment. }Offline-alignment of LLMs has mainly been achieved using preference datasets of the form  $(x_i, y_i^w, y_i^l)_{i=1}^N$, where we are given a prompt $x_i$, and a pair of prompt completions (or \emph{generations}) $(y_i^w, y_i^l)$ with $y_i^w$ the preferred generation and $y_i^l$ the dis-preferred one. Most objectives in the RLHF literature, such as DPO \citep{rafailov2023direct}, IPO \citep{azar2023general} or SLiC \citep{zhao2023slichf}, can be described and subsumed by the following loss:
\begin{equation*}
    \mathcal{L}(\theta) = \frac{1}{n}\sum_{i=1}^n f\left(\beta \cdot \left(  \log\left(\frac{\pi_\theta(y_i^w|x_i)}{\piref(y_i^w|x_i)}\right) -  \log\left(\frac{\pi_\theta(y_i^l|x_i)}{\piref(y_i^l|x_i)}\right)  \right) \right),
\end{equation*}
where $\beta$ is a scalar, $n$ a batch size, $\pi_\theta$ is the parameterised policy and $\piref$ a reference policy, typically obtained after a first step of pre-training and supervised fine-tuning. $f$ is a scalar function; each choice of function $f$ results in a specific given algorithm: for example,  $f(z) = \log\left(1+\exp(-z)\right)$ for DPO; $f(z) = \max\left(0,1-z \right)$ for SLiC; $f(z) = \left(z -1\right)^2 $ for IPO, among other possible alternatives as discussed in \citep{tang2024generalized}.

We argue that preference datasets are expensive to build and do not occur naturally in the wild. Most data coming from user logs is not collected pairwise, but instead comes in the form of a single trajectory.

\textbf{Single-trajectory setting. } Formally, we consider \emph{single-trajectory} datasets of the form $(x_i, y_i, r_i)_{i=1}^N$ where $x_i$ is a prompt, $y_i$ a generation and $r_i$ a scalar reward, collected by some unknown behavior policy.  As a simple example to model the thumbs-up vs. thumbs-down response, we can set a binary reward where $r_i=1$ is for thumbs-up. This formulation can be understood as a special case of the offline RL setup \citep{levine2020offline} tailored to the contextual bandit case for RLHF. 

Perhaps surprisingly, few offline alignment methods consider this setting with the exception of Kahneman-Tversky Optimization (KTO) \citep{kto}. KTO is derived from principles related to utility and prospect theory \citep{kahnemanprospect} that build upon the notion of human risk aversion. KTO also makes strong simplifying assumptions, which as we will show, biases the method to produce suboptimal policies. In contrast, we are interested in deriving a simple, general purpose and performant algorithm, without strong dependency on mathematical assumptions on risk preference or utility.

\paragraph{Online vs. offline algorithms.}
Whilst a natural idea might be to try and use \emph{online} RL (like in \citep{calandriello2024human,guo2024direct}), this would require one to either generate new prompt completions online, or to correct for the distribution of the online policy by using importance sampling. The latter in turn would bring its own set of challenges, e.g., high variance in the importance sampling ratios. In order to circumvent the associated difficulties, we consider the offline setting instead. Additionally, offline RL brings orthogonal benefits of its own, such as simplicity and computational efficiency. Therefore in the following, we present an offline, sound and practical method that approximates the optimal closed-form policy. 

\section{Direct Reward Optimisation (DRO)}
\label{sec:method}

In this section we present the main contribution of this work. We start with some theoretical background on the policy optimisation setting. We then introduce the \algo~objective and discuss a few important theoretical properties, followed by the design of practical algorithms.

\subsection{KL regularised policy optimisation with single-trajectory data}
\label{sec:generaltheory}

Complementary to the pairwise preference setting highlighted above, we now consider the \emph{single trajectory} case. We want to approximate the optimal Kullback-Leibler (KL) regularised policy $\pi^*$, following the canonical formulation of RLHF \citep{christiano2017deep}:
\beq
\pi^*(x)\eqdef \arg\max_\pi \quad \E_{x\sim\rho, y\sim\pi(\cdot|x)}[r(x,y) - \tau \cdot \KL(\pi(\cdot|x) \; || \; \piref(\cdot|x)) ],
\label{eq:setup}
\eeq
where $\piref$ is some initial reference policy, such as the policy obtained after pretraining and supervised fine-tuning. We then have that, necessarily,
\beq \label{eq:reg_solution}
\pi^*(y|x)=\frac{\piref (y|x) e^{\frac 1\tau r(x,y)}}{e^{\frac 1\tau V^*(x)}},
\eeq
where $V^*(x) \eqdef \tau \log \E_{y\sim\piref(\cdot|x)}[e^{\frac 1\tau r(x,y)}]$ is a function that depends on the regulariser $\piref$. Importantly, the normalisation constant, also called partition function, that normalizes the numerator in \eqref{eq:reg_solution} can be explicitly written as $Z=\exp{\frac{1}{\tau}V^*(x)}$. The denominator log-sum-exp as a value function has already appeared in the literature about \emph{soft} reinforcement learning \citep{ziebart2008maximum,haarnoja2017reinforcement,richemond2017short,schulman2018equivalence}. Unlike KTO, which assumes a constant partition function $Z$ for each prompt of the batch, we do not make any assumptions on the form of $Z$ or $V$. Unlike DPO or IPO, where cancellation of the partition function happens due to the difference of rewards in the Bradley-Terry preference model, we do not assume any functional form for $V$. For technically minded readers, we expand on these points in Appendix \ref{sec:appendix-technicalities}. We also note that the form of the partition function is intuitive, if we consider the Legendre-Fenchel conjugate \citep{Bauschke2011ConvexAA} of the KL regulariser in \eqref{eq:setup}. 

\paragraph{The \algo~objective.} Now, we rearrange the optimality condition from above that holds jointly with $\pi=\pi^*,V=V^*$,
\beq \label{eq:setup2}
 r(x,y) - V(x) = \tau \log \frac{\pi(y|x)}{\piref(y|x)}
\eeq
Multiple objectives (and subsequently algorithms) can be derived in order to enforce or approximate this condition, such that we refer to \algo~as a framework encompassing a plurality of algorithms. We can focus in particular on defining the following \algo~loss for any pair of policy and value functions $(\pi, V)$:
\beq
{\cal L}_{\texttt{\algo}}(\pi,V) \eqdef \frac 12 \E_{x\sim\rho, y\sim\mu(\cdot|x)} \left[ \left(  r(x,y) -V(x) - \tau \log\frac{\pi(y|x)}{\piref(y|x)} \right)^2\right].
\label{eq:DRO-V_loss}
\eeq

We begin with an existence and unicity result for the optimum of ${\cal L}_{\texttt{\algo}}$, thereafter denoted as $\mathcal{L}$.
\begin{theorem}\label{th:existence-unicity}
$(\pi^*, V^*)$ is a global optimum of the loss ${\cal L}(\pi,V)$.
In addition, assuming the supports of $\mu$ and $\piref$ coincide, (i.e., for all $x\in\supp \rho$, $\supp(\mu(\cdot|x))=\supp(\piref(\cdot|x))$), then $(\pi^*, V^*)$ is the unique global optimum of the loss ${\cal L}(\pi,V)$.
\end{theorem}

\begin{proof}
From the definition of $\pi^*$ and $V^*$, we have that ${\cal L}(\pi^*,V^*)=0$. Now notice that ${\cal L}(\pi,V)$ is non-negative since it is a sum of quadratic terms $t(x,y)^2$ where $t(x,y)\eqdef r(x,y) - V(x) -\tau \log\frac{\pi(y|x)}{\piref(y|x)}$. Thus $(\pi^*, V^*)$ is a global optimum of ${\cal L}(\pi,V)$. 

Now let us prove that it is unique. Assume there is another global optimum $(\tilde\pi,\tilde V)$ such that ${\cal L}(\tilde \pi,\tilde V)=0$. This means that for all $x\in \supp(\rho)$ and $y\in\supp \mu(\cdot|x)$, its $t(x,y)$-term is zero. Since the support of $\mu$ and $\piref$ coincide, we have that $\forall x\in \supp(\rho), \forall y\in\supp \piref(\cdot|x)$,
$$ r(x,y)-\tilde V(x) -\tau \log\frac{\tilde \pi(y|x)}{\piref(y|x)}  =0, 
$$
from which we deduce that 
$$\tilde \pi(y|x)=\frac{\piref(y|x) e^{\frac 1\tau r(x,y)}}{e^{\frac 1\tau \tilde V(x)}}.$$
But since $\tilde \pi(\cdot|x)$ is a probability distribution, we must have that $\tilde V(x)=\tau \log \sum_y \piref(y|x) e^{\frac 1\tau r(x,y)}$. Thus $\tilde V(x)=V^*(x)$ and $\tilde \pi(y|x)=\pi^*(y|x)$ for all $x\in\supp \rho$,  $y \in \supp(\piref(\cdot|x))$.
\end{proof}

\paragraph{Remarks.}
The significance of Theorem 1 lies in the fact that, though $\pi^*$ and $V^*$ are clearly related via $\log \pi^\ast(y|x) + V^*(x) = \log \pi_\texttt{ref}(y|x) + \frac{1}{\tau}r(x,y)$, there is \emph{no} need to account for that connection during optimisation. We can optimize $\pi$ and $V$ independently and end up finding the optimum. This can greatly simplify algorithmic design in practice, as we will see shortly. 

\paragraph{Role and necessity of the $V$ function. } The loss function requires a joint optimisation over $(\pi,V)$, which means separately parameterising a value function $V$ in addition to the policy $\pi$ so as to be theoretically sound. Thus, \algo~stands in strong contrast to alternative algorithms such as KTO \citep{kto}, which uses a policy loss only without a value function. One natural question arises as to whether it is possible to convert the loss function jointly over $(\pi,V)$ into a loss over $\pi$ alone. This is an idea we explore in Appendix~\ref{sec:stipo-novalue} where we optimise $V$ fully before updating $\pi$.

\subsection{Approximation error}
\label{sec:approx-error}

As alluded to before, we can interpret $V$ as a value function in conjunction with the  policy $\pi$. When fixing $\pi$ we can define the minimising value function $V^\pi= \arg\min_V \L(\pi,V)$. This value function, when fixed to $V=V^\pi$, provides learning signal to policy improvement. Indeed,  since  $(\pi,V)$ can be independently optimised as shown before, we can see that the optimal policy $\pi^*$ is obtained by solving the optimisation problem $\pi^*=\arg\min_\pi \L(\pi,V^\pi)=\arg\min_\pi \arg\min_V \L(\pi,V)$.

However, since in practice the learning of $V\approx V^\pi$ is approximate, the optimisation of $\pi$ against such an approximate value function would induce errors. Formally, we consider what happens when $\pi$ is optimised against a fixed $V$ rather than a well-learned $V^\pi$. The approximate nature of $V$ induces error in $\pi\approx \pi^*$, which we characterise below.

\begin{proposition}\label{prop:approx-error}
Consider an approximation $V(x)$ of the value function $V^\pi=\arg\min_V \L(\pi,V)$: 
$$V^\pi(x)= \E_{y\sim \mu(\cdot|x)} \left[ r(x,y) - \tau \log\frac{\pi(y|x)}{\piref(y|x)} \right]$$ 
and for this function kept fixed, let us optimise the loss $\pi\mapsto \L(\pi, V)$ with respect to the policy only. Let $\pi_V = \arg\min_\pi \L(\pi, V)$. Then we have that $\pi_V$ satisfies the following equation:
\begin{equation}\label{eq:first.claim}
    \pi_V(y|x) \propto \piref(y|x) e^{\frac 1\tau \left[r(x,y) - \frac{\pi_V(y|x)}{\mu(y|x)}\left(V^{\pi_V}(x) - V(x) \right) \right]}.
\end{equation}

In addition we have that for all $x,y$,
\beq\label{eq:bound.optimality}
\left|\log\frac{\pi_V(y|x)}{\pi^*(y|x)}\right|
\leq \frac 2\tau  \max_y \left| \left( V^{\pi_V}(x) - V(x)\right) \left( 1 - \frac{\pi_V(y|x)}{\mu(y|x)}\right) \right|.
\eeq
\end{proposition}
\begin{proof}
We delay the full technical proof to the Appendix, see section \ref{sec:appendix-proofs}.
\end{proof}
Two remarks are in order here:
\begin{itemize}[leftmargin=*]
\item \eqref{eq:bound.optimality} says that if our approximation $V$ is close to the value function of $\pi_V$ (in some sense, the "Bellman residual" $V-V^{\pi_V}$ of $V$ is small), then $\pi_V$ is close to $\pi^*$.
\item In particular, if we consider the best state-independent baseline $V^\pi_c = \arg\min_{V:V(x)\equiv c} {\cal L}(\pi,V)= \E_{x\sim \rho}[V^{\pi}(x)]$ (this could be estimated using a large enough batch size), we see that we would recover the optimal policy $\pi^*$ if $V^{\pi^*}(x)=V^{\pi^*}_c$ for all $x$.
\end{itemize}

\subsection{Practical Implementation}
\label{sec:algorithms}

We consider a parametric policy $\pi_\theta$ and a parametric function $V_\phi$ with neural network parameters $\theta,\phi$ respectively. We assume access to data in the form of tuples ${(x_i,y_i,r_i)_{1\leq i\leq n}}$ where $x_i\sim\rho$ are prompts, $y_i\sim \mu(\cdot|x_i)$ are prompt-conditional generations associated with a model policy $\mu$, and $r_i = r(x_i,y_i)$ are scalar reward functions for the individual prompt-completion combination $(x_i, y_i)$. We perform gradient descent, both on $\theta$ and $\phi$, to minimise the empirical loss:
$$\hat {\cal L}(\theta,\phi) \eqdef \frac 12 \sum_{i=1}^n \left( r(x_i,y_i)  -V_\phi(x_i) -\tau \log\frac{\pi_\theta(y_i|x_i)}{\piref(y_i|x_i)}\right)^2.
$$
We can define explicitly the gradient w.r.t.~$\phi$ and $\theta$: the gradient w.r.t.~the parameter $\phi$ of the value function is:
\begin{equation}\label{eq:stipov-value}
\nabla_\phi \hat{\mathcal{L}}(\theta, \phi) =   \sum_{i=1}^n\left(     V_{\phi}(x_i) - r(x_i, y_i) +\tau\log\left(\frac{\pi_\theta(y_i|x_i)}{\piref(y_i|x_i)}\right) \right) \nabla_\phi V_\phi(x_i),
\end{equation}
and the policy gradient:
\begin{align}\label{eq:stipov-policy}
\nabla_\theta \hat {\cal L}(\theta,\phi) &= -\tau \sum_{i=1}^n\Big(\nabla_\theta\underbrace{ \log \pi_\theta(y_i|x_i)\left( r(x_i,y_i)- V_\phi(x_i)\right) }_{\texttt{policy optimisation loss}} - \frac{\tau}{2} \nabla_\theta \underbrace{\Big( \log\frac{\pi_\theta(y_i|x_i)}{\piref(y_i|x_i)}\Big)^2}_{\texttt{$\ell_2$-regularisation loss}}\Big).
\end{align}
The policy optimisation loss is similar to a standard policy gradients RL algorithm, however it is important to notice a few key differences and connections between the loss above and policy gradient algorithms.

\textbf{Understanding the value function $V$. } First, when considering a fixed policy $\pi$, learning the value function $V_\phi$ corresponds to learning the function $V^{\pi}$ whose approximation error we discussed earlier in Section \ref{sec:approx-error}.

Second, the value function $V_\phi(x_i)$ which is subtracted from the reward $r(x_i,y_i)$ is not simply a baseline used for the purpose of variance reduction as is usually the case in RL \citep{Sutton00policy,konda1999actor}. Indeed removing this term (or replacing it by any function $V(x_i)$) would bias our policy gradient estimate. The reason for that being that the samples $y_i$ are off-policy, i.e., drawn from $\mu$ and not from $\pi_\theta$, thus in general, $\E_{y\sim\mu(\cdot|x_i)}[\nabla_\theta \log \pi_\theta(y|x_i) V(x_i)]\neq 0$ (whereas if on-policy we would have $\E_{y\sim\pi_\theta(\cdot|x_i)}[\nabla_\theta \log \pi(y|x_i) V(x_i)]= 0$). Thus it is important to maintain this value function estimate in the policy gradient.

\textbf{Offline regularisation. } Another difference compared to a usual regularised PG algorithm is the use of a $\ell_2$-regularisation loss instead of a KL-regularisation loss. Notice that these two regularisation losses do not lead to equivalent gradients for the reason that the samples $y_i$ are off-policy: indeed we have (see, e.g., also \citep{calandriello2024human,tang2024generalized}  for discussion of similar theoretical results)
\begin{align*}
\frac 12 \E_{y\sim\mu(\cdot|x_i)}\Big[ \nabla_\theta \Big( \log \frac{\pi_\theta(y|x_i)}{\piref(y|x_i)}\Big)^2\Big]\neq \nabla_\theta \KL(\pi_\theta(\cdot|x_i), \piref(\cdot|x_i)),
\end{align*}
unless the sampling is on-policy $\mu=\pi_\theta$, where both losses would would lead to the same gradient. Thus we could see the policy gradient update rule \eqref{eq:stipov-policy} as the natural extension of a usual on-policy regularised PG algorithm to the off-policy case.

\textbf{Policy learning rate rescaling. } Finally, in practice, we rescale the policy gradient \eqref{eq:stipov-policy} by multiplying the update by a factor of $1/\tau$. We found this works better empirically and we hypothesised that the global loss $\hat {\cal L}(\theta,\phi)$ may be ill-conditioned (as its sensitivities w.r.t.~the dimensions $\theta$ and $\phi$ are different, leading to a high condition number). We refer to this algorithm as the \emph{Direct Reward Optimisation with Value} algorithm, or \texttt{\algo-V}. \texttt{\algo-V} is described in Algorithm \ref{alg:algo-withvalue}.

\begin{algorithm}[h!]
   \caption{\texttt{D}irect \texttt{R}eward \texttt{O}ptimisation with \texttt{V}alue Algorithm (\texttt{\algo-V})}
   \label{alg:algo-withvalue}
\begin{algorithmic}
   \STATE {\bfseries Inputs:} A single-trajectory  dataset: $(x_i, y_i, r_i=r(x_i, y_i))_{i=1}^N$, a parameterised policy: $\pi_\theta$, a reference policy: $\piref$, a parameterised value function: $V_\phi$, a regularisation scalar $\tau$, a number of total steps $K$, a batch size $B$ and an optimiser.  

   \FOR{$k=1$ {\bfseries to} $K$}
   \STATE Sample uniformly a batch: $(x_i, y_i, r_i)_{i=1}^B$
   \STATE Compute gradient updates $\nabla_\theta \hat {\cal L}(\theta,\phi)$ and $\nabla_\phi \hat {\cal L}(\theta,\phi)$ as in  \eqref{eq:stipov-policy} and \eqref{eq:stipov-value}: 
   
   $$\nabla_\theta \hat {\cal L}(\theta,\phi) = -\frac{1}{B} \sum_{i=1}^B\Big(\nabla_\theta \log \pi_\theta(y_i|x_i)\left( r(x_i,y_i)- V_\phi(x_i)\right) - \frac{1}{2} \nabla_\theta \Big( \log\frac{\pi_\theta(y_i|x_i)}{\piref(y_i|x_i)}\Big)^2 \Big) $$
   
   $$\nabla_\phi \hat{\mathcal{L}}(\theta, \phi) = \frac{1}{B}  \sum_{i=1}^B\left(     V_{\phi}(x_i) - r(x_i, y_i) +\tau\log\left(\frac{\pi_\theta(y_i|x_i)}{\piref(y_i|x_i)}\right) \right) \nabla_\phi V_\phi(x_i)$$
   
   \STATE Update the policy parameters: $\theta \gets \texttt{UpdateOptimiser}(\theta, \nabla_\theta \hat {\cal L}(\theta,\phi))$
   \STATE Update the value parameters: $\phi \gets \texttt{UpdateOptimiser}(\phi, \nabla_\phi \hat {\cal L}(\theta,\phi))$

   \ENDFOR
   \STATE {\bfseries Outputs:} $\pi_\theta$
\end{algorithmic}
\end{algorithm}

\textbf{Offline optimisation. } The dataset of (prompt, completion, reward) triplets remains static during optimisation. This is because we do not use the current "online" policy, parameterised by the most recent $\theta$ parameters, to regenerate completions for a given prompt. In this regard, our optimisation is performed like in offline reinforcement learning, where taking new actions in the environment is structurally prohibited. We note that our method does not require the training of an additional reward model. The distinction between \emph{offline} and \emph{online} procedures is important in RL(HF), since they can give rise to different gradients and correspond to very different theoretical justifications, as studied in e.g. \citet{calandriello2024human,guo2024direct}. Given the additional $\tau$ regularisation term appearing in \eqref{eq:stipov-policy}, we describe \algo~as offline, regularised reinforcement learning for large language model alignment.

\textbf{Neural network implementation. } It is natural to wonder whether $\theta$ and $\phi$ need be separate or whether \emph{parameter sharing} can occur for efficient learning. Perhaps counter-intuitively, we found that using two separate networks, one for $\pi_\theta$ and one for $V_\phi$, was beneficial empirically, compared to using policy logits as value outputs. Another implementation decision, this time strictly related to $V_\phi$, is whether to use a single value per batch or a value per token. We found that using a single value per batch hurts performance. In both instances parameter sharing is detrimental.

As such, by default, we implemented \texttt{DRO-V} using two networks, as well as multiple values across the batch. We will return to these points and ablate these design choices in the Experiments section next. 

\section{Experiments}
\label{sec:experiments}

We now present our empirical results on finetuning LLMs using \texttt{DRO-V}.

\textbf{Models and task dataset. } In all that follows, we perform our test experiments using the \emph{UltraFeedback} dataset described in ~\citet{cui2023ultrafeedback}. We preprocess the data to create an offline dataset of triplets (prompt, completion, reward) where the prompt might be shared across multiple triplets. We also normalize the dataset such that the rewards have mean 0 and variance 1 across the dataset.

We use T5 large language models~\citep{JMLR:v21:20-074}, a family of auto-regressive transformers with an encoder-decoder architecture, in order to train all our models. The details of the models architecture and software implementation can be found in~\citep{roberts2022t5x}. Furthermore, our checkpoints are initialised from instruction finetuning according to the FLAN recipe \citep{chung2022scaling}. We denote this initialisation policy ($\piref$ as earlier), as the SFT, for supervised finetuning policy. Depending on the experiment, we use either a \emph{large} (L) or an \emph{extra-large} (XL) encoder-decoder model. The large model sports $770M$ parameters, whereas the XL model has $3B$ parameters.

\textbf{Evaluation. } Following now standard practice \citep{zheng2023judging}, our evaluation pipeline consists in automated evaluation; specifically, side-by-side comparison. We use the PaLM2~\citep{anil2023palm} LLM as a judge. Given a test set of prompts, for each pair of trained policies, we sample completion responses, and then ask PaLM2 to judge which one is better. The format of the evaluation prompt we use for side-by-side comparison is as follows:

\texttt{\small In this task, you will be provided with an instruction and two responses. Your job is to assess the helpfulness and fulfillment of two responses A and B.}

\texttt{\small Instruction: {article} response A (left):{summary1} response B (right):{summary2}}

For each model, we record a checkpoint every $2,000$ training steps, before selecting the best checkpoint across the training curve as determined by side-by-side comparison against the SFT policy.  We then use this best checkpoint as representative of an algorithm or set of parameters.

\textbf{Compute and hyperparameters. } For compute, we use version $5$ ('\emph{v5e}') Tensor Processing Units \citep[TPUs; ][]{Jouppi2023TPUVA}, in the cloud. We train \emph{large} encoders in configurations of $4 \times 4$ devices and \emph{XL} encoders in configurations of $4 \times 8$ devices. With this computational setup we obtain speeds of around $0.5$ training steps per second ($21$ hours per $40,000$ steps) for the \emph{large} encoders, and $0.1$ training steps per second ($2$ days per $20,000$ steps) for the \emph{XL} encoders. We run our experiments with default learning rate $1$e-$4$ both for the value and the policy networks, and a default total of $40,000$ training steps for $L$ models and $20,000$ training steps for $XL$ models, using a batch size of $32$. The optimiser we use is AdaFactor \citep{Shazeer2018AdafactorAL} with a decay rate value of $0.8$. For the learning rate, we employ $150$ linear warmup steps.

\textbf{KTO baseline. } We compare our algorithm primarily to the strong Kahneman-Tversky Optimization (KTO) baseline of \citet{kto}. We need a criterion to decide which sample completions are desirable ('thumbs up') or not. We choose to do so per-batch, based on whether each attached scalar reward is greater or less than the average reward seen over the minibatch. We use the same KL divergence estimator on the batch as they do, and do not weigh the loss in any manner, since we are already compensating for batch statistics by computing the batchwise average reward. For both algorithms KTO and \texttt{DRO-V}, the $\tau$ regularisation factor (denoted by $\beta$ in \citet{kto}) is held constant throughout training.

\subsection{Results and Ablations}
\label{exp:results-ablations}

Here we show our empirical results as well as perform ablations and study the impact of hyperparameters on performance.

\textbf{Empirical results. }
\texttt{DRO-V} outperforms KTO in side-by-side comparison, both for the T5-L and XL encoders. In the interest of fair comparison, we optimise our baseline as much as possible and present the best KTO results we obtained after search over three values of the $\tau$ regularisation parameter, $\tau \in \{0.1, 1.0, 5.0 \}$ ($\tau=5.0$ being optimal). By contrast we show our \texttt{DRO-V} results without such search, simply setting $\tau$ to a default intuitive value of $1.0$. In this section only, we give standard deviation estimates, thanks to computing averages of scores obtained by side-by-side comparison over $5$ evaluation folds of $1,000$ test prompts each:

\begin{figure}[!hb]\label{fig:generalperf}
\centering
\begin{minipage}{0.4\textwidth}
\centering

\begin{tabular}{c|c}
\toprule
Side-by-side Winrate & (first over second) \\
\midrule
\texttt{DRO-V} vs SFT: &  $78.9 \% \pm 0.3\%$ \\
KTO vs SFT: &  $67.5 \% \pm 0.7\%$  \\
\midrule
\texttt{DRO-V} vs KTO: &  $63.4 \% \pm 1.0\%$\\
\bottomrule
\end{tabular}

\caption{Winrates with T5-L encoders.}
\end{minipage}
\hfill
\begin{minipage}{0.45\textwidth}
\centering

\begin{tabular}{c|c}
\toprule
Side-by-side Winrate & (first over second) \\
\midrule
\texttt{DRO-V} vs SFT: & $81.5 \% \pm 1.0 \%$ \\
KTO vs SFT: & $78.2 \% \pm 0.7 \%$ \\
\midrule
\texttt{DRO-V} vs KTO: & $57.5 \% \pm 0.8 \%$ \\
\bottomrule
\end{tabular}

\caption{Winrates with T5-XL encoders.}
\end{minipage}
\end{figure}

We see significant empirical outperformance of \texttt{DRO-V}. Both on T5-L and T5-XL encoders, \texttt{DRO-V} clearly outperforms the SFT ($78.9\%$ and $81.5\%$, respectively). Most importantly, \texttt{DRO-V} also wins over the KTO baseline in direct comparison ($63.4\%$ and $57.5\%$ respectively). We also give qualitative examples of this difference, with sample completions from both algorithms, in Appendix \ref{sec:appendix-qualitative} where we demonstrate how these numbers correspond to more helpful and focused prompt completions.

We now study ablations over hyperparameters and architecture. In all that follows, we retain the same experimental protocol as above, and here use T5 \emph{large} encoders exclusively.

\textbf{Impact of the learning rate for policy and value. }
We begin by examining the impact of the learning rate on downstream performance. We study two cases: first varying jointly the learning rate of the policy and the value network, and second, switching the learning rate of the value network only. We pick our learning rates to be one of $1$e-$5$, $5$e-$5$, $5$e-$4$, or the default $1$e-$4$. Results are presented in Figure \ref{fig:learningrates}. For the value only sweep, we observe that changing the value learning rate parameter alone yields a small impact. Thus the policy learning rate rescaling factor, $1/\tau$ (Equation \ref{eq:stipov-policy}) is all the more important. We do also note a small yet monotonic improvement in learning the value $V$ faster than the policy. Overall, the performance of \texttt{DRO-V} remains very stable within an order of magnitude change for learning rates.

\begin{figure}[!ht]
\centering
\begin{minipage}{0.99\textwidth}
\centering
\begin{tabular}{c|ccccc}
\toprule
\texttt{DRO-V} learning rate value & $1$e$-5$ & $5$e-$5$ & $1$e-$4$ & $5$e-$4$ \\
\midrule
Joint LR : Winrate vs SFT & $73.7 \%$ & $78.7 \%$ & $\bf{78.9 \%}$ & $78.4 \%$ \\
LR, value $V$ only : Winrate vs SFT & $76.8 \%$ & $78.0 \%$ & $78.9 \%$ & $\bf{79.1 \%}$ \\
\bottomrule
\end{tabular}

\end{minipage}
\caption{\emph{Top line:} Impact of varying jointly the $\pi$ and $V$ learning rate parameter. \emph{Bottom line:} Impact of value function $V$ learning rate parameter only.}
\label{fig:learningrates}
\end{figure}

\textbf{Impact of regularisation parameter $\tau$. } Similarly, we vary the strength of regularisation parameter $\tau$ both for \texttt{DRO-V} and KTO. We pick between three values: $0.1$, $1.0$ and $5.0$. Results are presented in Figure \ref{fig:tauparam}. We see substantial variation in performance due to this parameter. A regularisation parameter of $1.0$, an intuitive value, is actually best for \texttt{DRO-V}. On the other hand, we found it far from optimal for KTO and picked the best value, $5.0$, instead. These experiments were performed using T5-L encoders, and we re-used those $\tau$ choices for T5-XL experiments as well.

\begin{figure}[!ht]
\centering
\begin{minipage}{0.45\textwidth}
\centering
\begin{tabular}{c|cccc}
\toprule
\texttt{DRO-V $\tau$} & $0.1$ & $1.0$  & $5.0$ \\
\midrule
Winrate vs SFT  & $70.5 \%$ & $\bf{78.9 \%}$ & $76.6 \%$ \\
\bottomrule
\end{tabular}
\end{minipage}
\hfill
\begin{minipage}{0.45\textwidth}
\centering
\begin{tabular}{c|cccc}
\toprule
KTO $\tau$ & $0.1$ & $1.0$ & $5.0$ \\
\midrule
Winrate vs SFT & $63.5 \%$ & $61.9 \%$ & $\bf{67.5 \%}$ \\
\bottomrule
\end{tabular}
\end{minipage}
\caption{\emph{Left:} Impact of $\tau$ parameter on \texttt{DRO-V}. \emph{Right:}  Impact of $\tau$ parameter on KTO.}
\label{fig:tauparam}
\end{figure}

\textbf{Impact of parameter sharing. } Finally, we investigate the quantitative impact of parameter sharing for \texttt{DRO-V}, as exposed in Section \ref{sec:algorithms}. We jointly study the performance of the single or double network version, as well as whether to use a single value number per batch or not. For computational reasons, these experiments at performed using $10,000$ steps of training only (and therefore slightly undertrained compared to our main T5-L $78.9\%$ result). Results are presented Figure \ref{fig:ablation-paramshare}, with the full comparison matrix in Appendix \ref{sec:appendix-experiments}. The impact of full parameter sharing is material, with most of the hit coming from the single or double network choice. However, we also observe that when using two networks for $\pi$ and $V$, there are significant gains in not using a single value per batch ($76.6\%$ against $72.1\%$ winrate vs SFT, a difference of $4.5\%$, confirmed in direct side-by-side comparison of those two variants yielding $54.9\%$ in favour of the multiple value version). These observations help explain some of the outperformance of our method.

\begin{figure}[!ht]
\centering
\begin{minipage}{0.99\textwidth} 
\centering

\begin{tabular}{c|cccc}
\toprule
Parameter Sharing Variant & \makecell[c]{Double Net\\Single Value} & \makecell[c]{Single Net\\Single Value} & \makecell[c]{Single Net\\Multiple Values} & \makecell[c]{Double Net\\Multiple Values} \\
\midrule
Winrate vs SFT & $72.1 \%$ & $57.6 \%$ & $55.5 \%$ & $\bf{76.6 \%}$ \\
\bottomrule
\end{tabular}

\caption{Parameter sharing variants. \texttt{DRO-V} Winrate vs SFT.}
\label{fig:ablation-paramshare}
\end{minipage}
\end{figure}

\section{Related work}
\label{sec:related}

\textbf{Human Feedback in Reinforcement Learning. } Integrating human feedback into reinforcement learning, as introduced by \citet{christiano2017deep}, has rapidly grown to be considered essential for improving the practical utility of LLMs and mitigating their epistemic risk \citep{hannigan2024}. RLHF may not always yield direct improvements in benchmark performance \citep{Touvron2023Llama2O}, but it significantly enhances human-centric applications like dialogue systems \citep{nakano2021webgpt, InstructGPT} and extends to non-human-centric tasks too like MuJoCo physics \citep{yuan2024unirlhf} and robotics \citep{gao2024rebel}. Precursors to RLHF were implemented with Deep Q Networks \citep{Mnih13Playing} and Actor-Critic algorithms \citep{Mnih2016asynchronous, glaese2022improving}, and RLHF itself was debuted with proximal policy optimisation \citep{PPO}. Given the complexity of RLHF techniques and algorithms used \citep{casper2023open}, the research community is revealing surprising benefits of simpler strategies like REINFORCE \citep{ahmadian2024basics}, sequence likelihood \citep{zhao2023slichf}, and ranking approaches \citep{Dong2023RAFTRR, Yuan2023RRHFRR}. 

\textbf{Advancements in Policy Optimisation. } Assuming a Bradley-Terry model \citep{bradley1952rank} for human reward modeling enables the RLHF problem to be reformulated as a supervised learning task \citep{rafailov2023direct}; this may provide greater training stability \citep{zhao2023beyond} and data efficiency, although this point is debated \citep{xu2024dpo}. Recent research has focused on enhancing the direct preference optimization (DPO) methods for scalability \citep{tunstall2023zephyr, Ivison2023CamelsIA} and safety \citep{liu2024enhancing} while expanding their mathematical underpinnings \citep{azar2023general, Wang2023BeyondRK, exo}. Other work constrains DPO’s contextual focus \citep{zeng2024tokenlevel} or explores alternatives to preference models such as Nash equilibria \citep{munos2023nash, rosset2024direct}. Despite these advancements, both DPO and RLHF encounter challenges such as reward hacking \citep{Pang2022RewardGI, Skalse2022DefiningAC}, length bias \citep{park2024disentangling}, and overoptimisation \citep{Amodei2016ConcretePI, Pan2022TheEO}, which can in turn lead to under-regularised models \citep{Gao2022ScalingLF, Singhal2023ALW} and objective mismatch \citet{kirk2023understanding}. While techniques like ensembling can alleviate issues of alignment \citep{Eisenstein2023HelpingOH, Rame2024WARMOT} and overoptimisation \citep{Wortsman2022ModelSA, Coste2023RewardME}, they may be overlooked due to their computational intensity \citep{zhang2024improving}.

\textbf{Different data types. } RLHF classically focuses on a choice that human raters make between two outputs produced from a single input to the policy.  Collecting this data type can be costly, motivating the use of cheaper data, such as upvote/downvote point-wise data \citep{kto}, or joint preferences from different questions \citep{bansal2024comparing}. Such data can also be noisy, motivating factual augmentation \citep{sun2023aligning}, fine-grained responses \citep{wu2023finegrained}.

\textbf{Online vs. Offline Learning and Alignment through Self-Play. } The practical distinction between online and offline methods \citep{jaques2019way} appears more relevant than the reinforcement versus supervised learning dichotomy. Online policies risk deviating from the original data distribution, causing shifts and potential issues \citep{Zhuang2021ConsequencesOM,Shin2023BenchmarksAA}. Alternatively, alignment can be achieved through self-play in a two-player game framework \citep{munos2023nash, Swamy2024AMA}, encompassing both online and offline settings and allowing smooth transitions between them. Similarly, iterative DPO has demonstrated alignment improvement \citep{yuan2024selfrewarding}, building upon established techniques like reinforcement learning from AI feedback \citep{Bai2022ConstitutionalAH, lee2023rlaif}.

\section{Conclusion and limitations}

We have introduced \algo, a new framework for aligning LLMs in the setting of single-trajectory datasets, where for each prompt a single completion associated to a scalar reward is available. \algo~moves away from the traditional preference setting in RLHF. In doing so it not only makes the training of an explicit reward model redundant, but also and most importantly enables leveraging potentially orders of magnitude more data coming from user feedback instead of raters'. \algo~is theoretically principled since it learns both a policy and a value. Thus it does not rely on any simplifying mathematical assumptions, such as a uniform value function on the batch, or the cancellation of the partition function that typically underpins RLHF methods. Using T5 \emph{large} and \emph{XL} text encoders, we have shown that these properties translate into strong performance on the \emph{UltraFeedback} dataset when compared to Kahneman-Tversky optimisation. However, our empirical study is limited, both in terms of number of tasks and scale. Further work is required to more broadly establish the performance gains that our approach provides when considering the largest language models, as it is able to leverage large amounts of user-generated data.

\section*{Acknowledgements}
We would like to thank the whole Google DeepMind team for providing the infrastructure to make this work possible. In particular, we would like to thank Matt W. Hoffman, Bobak Shahriari, Nikola Momchev, Sertan Girgin and Piotr Stanczyk for their support in building the coding infrastructure and Doina Precup and Olivier Bachem for their guidance and continual support. 

\newpage

\bibliography{biblio}

\begin{thebibliography}{85}
\providecommand{\natexlab}[1]{#1}
\providecommand{\url}[1]{\texttt{#1}}
\expandafter\ifx\csname urlstyle\endcsname\relax
  \providecommand{\doi}[1]{doi: #1}\else
  \providecommand{\doi}{doi: \begingroup \urlstyle{rm}\Url}\fi

\bibitem[Ahmadian et~al.(2024)Ahmadian, Cremer, Gallé, Fadaee, Kreutzer,
  Pietquin, Üstün, and Hooker]{ahmadian2024basics}
Arash Ahmadian, Chris Cremer, Matthias Gallé, Marzieh Fadaee, Julia Kreutzer,
  Olivier Pietquin, Ahmet Üstün, and Sara Hooker.
\newblock Back to basics: Revisiting reinforce style optimization for learning
  from human feedback in llms.
\newblock \emph{arxiv preprint arXiv:2402.14740}, 2024.

\bibitem[Amodei et~al.(2016)Amodei, Olah, Steinhardt, Christiano, Schulman, and
  Man{\'e}]{Amodei2016ConcretePI}
Dario Amodei, Chris Olah, Jacob Steinhardt, Paul Christiano, John Schulman, and
  Dan Man{\'e}.
\newblock Concrete problems in {AI} safety.
\newblock \emph{arXiv}, 2016.

\bibitem[Anil et~al.(2023)Anil, Dai, Firat, Johnson, Lepikhin, Passos, Shakeri,
  Taropa, Bailey, Chen, Chu, Clark, Shafey, Huang, Meier-Hellstern, Mishra,
  Moreira, Omernick, Robinson, Ruder, Tay, Xiao, Xu, Zhang, Abrego, Ahn,
  Austin, Barham, Botha, Bradbury, Brahma, Brooks, Catasta, Cheng, Cherry,
  Choquette-Choo, Chowdhery, Crepy, Dave, Dehghani, Dev, Devlin, Díaz, Du,
  Dyer, Feinberg, Feng, Fienber, Freitag, Garcia, Gehrmann, Gonzalez, Gur-Ari,
  Hand, Hashemi, Hou, Howland, Hu, Hui, Hurwitz, Isard, Ittycheriah, Jagielski,
  Jia, Kenealy, Krikun, Kudugunta, Lan, Lee, Lee, Li, Li, Li, Li, Li, Lim, Lin,
  Liu, Liu, Maggioni, Mahendru, Maynez, Misra, Moussalem, Nado, Nham, Ni,
  Nystrom, Parrish, Pellat, Polacek, Polozov, Pope, Qiao, Reif, Richter, Riley,
  Ros, Roy, Saeta, Samuel, Shelby, Slone, Smilkov, So, Sohn, Tokumine, Valter,
  Vasudevan, Vodrahalli, Wang, Wang, Wang, Wang, Wieting, Wu, Xu, Xu, Xue, Yin,
  Yu, Zhang, Zheng, Zheng, Zhou, Zhou, Petrov, and Wu]{anil2023palm}
Rohan Anil, Andrew~M. Dai, Orhan Firat, Melvin Johnson, Dmitry Lepikhin,
  Alexandre Passos, Siamak Shakeri, Emanuel Taropa, Paige Bailey, Zhifeng Chen,
  Eric Chu, Jonathan~H. Clark, Laurent~El Shafey, Yanping Huang, Kathy
  Meier-Hellstern, Gaurav Mishra, Erica Moreira, Mark Omernick, Kevin Robinson,
  Sebastian Ruder, Yi~Tay, Kefan Xiao, Yuanzhong Xu, Yujing Zhang,
  Gustavo~Hernandez Abrego, Junwhan Ahn, Jacob Austin, Paul Barham, Jan Botha,
  James Bradbury, Siddhartha Brahma, Kevin Brooks, Michele Catasta, Yong Cheng,
  Colin Cherry, Christopher~A. Choquette-Choo, Aakanksha Chowdhery, Clément
  Crepy, Shachi Dave, Mostafa Dehghani, Sunipa Dev, Jacob Devlin, Mark Díaz,
  Nan Du, Ethan Dyer, Vlad Feinberg, Fangxiaoyu Feng, Vlad Fienber, Markus
  Freitag, Xavier Garcia, Sebastian Gehrmann, Lucas Gonzalez, Guy Gur-Ari,
  Steven Hand, Hadi Hashemi, Le~Hou, Joshua Howland, Andrea Hu, Jeffrey Hui,
  Jeremy Hurwitz, Michael Isard, Abe Ittycheriah, Matthew Jagielski, Wenhao
  Jia, Kathleen Kenealy, Maxim Krikun, Sneha Kudugunta, Chang Lan, Katherine
  Lee, Benjamin Lee, Eric Li, Music Li, Wei Li, YaGuang Li, Jian Li, Hyeontaek
  Lim, Hanzhao Lin, Zhongtao Liu, Frederick Liu, Marcello Maggioni, Aroma
  Mahendru, Joshua Maynez, Vedant Misra, Maysam Moussalem, Zachary Nado, John
  Nham, Eric Ni, Andrew Nystrom, Alicia Parrish, Marie Pellat, Martin Polacek,
  Alex Polozov, Reiner Pope, Siyuan Qiao, Emily Reif, Bryan Richter, Parker
  Riley, Alex~Castro Ros, Aurko Roy, Brennan Saeta, Rajkumar Samuel, Renee
  Shelby, Ambrose Slone, Daniel Smilkov, David~R. So, Daniel Sohn, Simon
  Tokumine, Dasha Valter, Vijay Vasudevan, Kiran Vodrahalli, Xuezhi Wang,
  Pidong Wang, Zirui Wang, Tao Wang, John Wieting, Yuhuai Wu, Kelvin Xu, Yunhan
  Xu, Linting Xue, Pengcheng Yin, Jiahui Yu, Qiao Zhang, Steven Zheng,
  Ce~Zheng, Weikang Zhou, Denny Zhou, Slav Petrov, and Yonghui Wu.
\newblock {PaLM} 2 technical report, 2023.

\bibitem[Azar et~al.(2023)Azar, Rowland, Piot, Guo, Calandriello, Valko, and
  Munos]{azar2023general}
Mohammad~Gheshlaghi Azar, Mark Rowland, Bilal Piot, Daniel Guo, Daniele
  Calandriello, Michal Valko, and R{\'e}mi Munos.
\newblock A general theoretical paradigm to understand learning from human
  preferences.
\newblock \emph{arXiv}, 2023.

\bibitem[Bai et~al.(2022{\natexlab{a}})Bai, Jones, Ndousse, Askell, Chen,
  DasSarma, Drain, Fort, Ganguli, Henighan, Joseph, Kadavath, Kernion, Conerly,
  El-Showk, Elhage, Hatfield-Dodds, Hernandez, Hume, Johnston, Kravec, Lovitt,
  Nanda, Olsson, Amodei, Brown, Clark, McCandlish, Olah, Mann, and
  Kaplan]{bai2022training}
Yuntao Bai, Andy Jones, Kamal Ndousse, Amanda Askell, Anna Chen, Nova DasSarma,
  Dawn Drain, Stanislav Fort, Deep Ganguli, Tom Henighan, Nicholas Joseph,
  Saurav Kadavath, Jackson Kernion, Tom Conerly, Sheer El-Showk, Nelson Elhage,
  Zac Hatfield-Dodds, Danny Hernandez, Tristan Hume, Scott Johnston, Shauna
  Kravec, Liane Lovitt, Neel Nanda, Catherine Olsson, Dario Amodei, Tom Brown,
  Jack Clark, Sam McCandlish, Chris Olah, Ben Mann, and Jared Kaplan.
\newblock Training a helpful and harmless assistant with reinforcement learning
  from human feedback.
\newblock \emph{arXiv}, 2022{\natexlab{a}}.

\bibitem[Bai et~al.(2022{\natexlab{b}})Bai, Kadavath, Kundu, Askell, Kernion,
  Jones, Chen, Goldie, Mirhoseini, McKinnon, Chen, Olsson, Olah, Hernandez,
  Drain, Ganguli, Li, Tran-Johnson, Perez, Kerr, Mueller, Ladish, Landau,
  Ndousse, Lukoiūtė, Lovitt, Sellitto, Elhage, Schiefer, Mercado, DasSarma,
  Lasenby, Larson, Ringer, Johnston, Kravec, Showk, Fort, Lanham,
  Telleen-Lawton, Conerly, Henighan, Hume, Bowman, Hatfield-Dodds, Mann,
  Amodei, Joseph, McCandlish, Brown, and Kaplan]{Bai2022ConstitutionalAH}
Yuntao Bai, Saurav Kadavath, Sandipan Kundu, Amanda Askell, John Kernion, Andy
  Jones, Anna Chen, Anna Goldie, Azalia Mirhoseini, Cameron McKinnon, Carol
  Chen, Catherine Olsson, Christopher Olah, Danny Hernandez, Dawn Drain, Deep
  Ganguli, Dustin Li, Eli Tran-Johnson, E~Perez, Jamie Kerr, Jared Mueller,
  Jeff Ladish, J~Landau, Kamal Ndousse, Kamilė Lukoiūtė, Liane Lovitt,
  Michael Sellitto, Nelson Elhage, Nicholas Schiefer, Noem'i Mercado, Nova
  DasSarma, Robert Lasenby, Robin Larson, Sam Ringer, Scott Johnston, Shauna
  Kravec, Sheer~El Showk, Stanislav Fort, Tamera Lanham, Timothy
  Telleen-Lawton, Tom Conerly, T.~J. Henighan, Tristan Hume, Sam Bowman, Zac
  Hatfield-Dodds, Benjamin Mann, Dario Amodei, Nicholas Joseph, Sam McCandlish,
  Tom~B. Brown, and Jared Kaplan.
\newblock Constitutional {AI}: Harmlessness from {AI} feedback.
\newblock \emph{arXiv}, 2022{\natexlab{b}}.

\bibitem[Bansal et~al.(2024)Bansal, Suvarna, Bhatt, Peng, Chang, and
  Grover]{bansal2024comparing}
Hritik Bansal, Ashima Suvarna, Gantavya Bhatt, Nanyun Peng, Kai-Wei Chang, and
  Aditya Grover.
\newblock Comparing bad apples to good oranges: Aligning large language models
  via joint preference optimization.
\newblock \emph{arxiv preprint arXiv: 2404.00530}, 2024.

\bibitem[Bauschke and Combettes(2011)]{Bauschke2011ConvexAA}
Heinz~H. Bauschke and Patrick~L. Combettes.
\newblock \emph{Convex Analysis and Monotone Operator Theory in Hilbert
  Spaces}.
\newblock Springer, 2011.

\bibitem[Bowman et~al.(2022)Bowman, Hyun, Perez, Chen, Pettit, Heiner,
  Luko{\v{s}}i{\=u}t{\.e}, Askell, Jones, Chen, et~al.]{bowman2022measuring}
Samuel~R Bowman, Jeeyoon Hyun, Ethan Perez, Edwin Chen, Craig Pettit, Scott
  Heiner, Kamil{\.e} Luko{\v{s}}i{\=u}t{\.e}, Amanda Askell, Andy Jones, Anna
  Chen, et~al.
\newblock Measuring progress on scalable oversight for large language models.
\newblock \emph{arXiv preprint arXiv:2211.03540}, 2022.

\bibitem[Bradley and Terry(1952)]{bradley1952rank}
Ralph~Allan Bradley and Milton~E. Terry.
\newblock Rank analysis of incomplete block designs: I. the method of paired
  comparisons.
\newblock \emph{Biometrika}, 39\penalty0 (3/4):\penalty0 324--345, 1952.

\bibitem[Calandriello et~al.(2024)Calandriello, Guo, Munos, Rowland, Tang,
  Pires, Richemond, Lan, Valko, Liu, Joshi, Zheng, and
  Piot]{calandriello2024human}
Daniele Calandriello, Daniel Guo, Remi Munos, Mark Rowland, Yunhao Tang,
  Bernardo~Avila Pires, Pierre~Harvey Richemond, Charline~Le Lan, Michal Valko,
  Tianqi Liu, Rishabh Joshi, Zeyu Zheng, and Bilal Piot.
\newblock Human alignment of large language models through online preference
  optimisation.
\newblock \emph{arxiv preprint arXiv:2403.08635}, 2024.

\bibitem[Casper et~al.(2023)Casper, Davies, Shi, Gilbert, Scheurer, Rando,
  Freedman, Korbak, Lindner, Freire, et~al.]{casper2023open}
Stephen Casper, Xander Davies, Claudia Shi, Thomas~Krendl Gilbert,
  J{\'e}r{\'e}my Scheurer, Javier Rando, Rachel Freedman, Tomasz Korbak, David
  Lindner, Pedro Freire, et~al.
\newblock Open problems and fundamental limitations of reinforcement learning
  from human feedback.
\newblock \emph{arXiv preprint arXiv:2307.15217}, 2023.

\bibitem[Chen et~al.(2020)Chen, Kornblith, Norouzi, and Hinton]{chen2020simple}
Ting Chen, Simon Kornblith, Mohammad Norouzi, and Geoffrey Hinton.
\newblock A simple framework for contrastive learning of visual
  representations.
\newblock \emph{arxiv preprint arXiv:2002.05709}, 2020.

\bibitem[Christiano et~al.(2017)Christiano, Leike, Brown, Martic, Legg, and
  Amodei]{christiano2017deep}
Paul~F. Christiano, Jan Leike, Tom Brown, Miljan Martic, Shane Legg, and Dario
  Amodei.
\newblock Deep reinforcement learning from human preferences.
\newblock In \emph{Advances in Neural Information Processing Systems}, 2017.

\bibitem[Chung et~al.(2022)Chung, Hou, Longpre, Zoph, Tay, Fedus, Li, Wang,
  Dehghani, Brahma, Webson, Gu, Dai, Suzgun, Chen, Chowdhery, Castro-Ros,
  Pellat, Robinson, Valter, Narang, Mishra, Yu, Zhao, Huang, Dai, Yu, Petrov,
  Chi, Dean, Devlin, Roberts, Zhou, Le, and Wei]{chung2022scaling}
Hyung~Won Chung, Le~Hou, Shayne Longpre, Barret Zoph, Yi~Tay, William Fedus,
  Yunxuan Li, Xuezhi Wang, Mostafa Dehghani, Siddhartha Brahma, Albert Webson,
  Shixiang~Shane Gu, Zhuyun Dai, Mirac Suzgun, Xinyun Chen, Aakanksha
  Chowdhery, Alex Castro-Ros, Marie Pellat, Kevin Robinson, Dasha Valter,
  Sharan Narang, Gaurav Mishra, Adams Yu, Vincent Zhao, Yanping Huang, Andrew
  Dai, Hongkun Yu, Slav Petrov, Ed~H. Chi, Jeff Dean, Jacob Devlin, Adam
  Roberts, Denny Zhou, Quoc~V. Le, and Jason Wei.
\newblock Scaling instruction-finetuned language models, 2022.

\bibitem[Coste et~al.(2023)Coste, Anwar, Kirk, and Krueger]{Coste2023RewardME}
Thomas Coste, Usman Anwar, Robert Kirk, and David~Scott Krueger.
\newblock Reward model ensembles help mitigate overoptimization.
\newblock \emph{arXiv}, 2023.

\bibitem[Cui et~al.(2023)Cui, Yuan, Ding, Yao, Zhu, Ni, Xie, Liu, and
  Sun]{cui2023ultrafeedback}
Ganqu Cui, Lifan Yuan, Ning Ding, Guanming Yao, Wei Zhu, Yuan Ni, Guotong Xie,
  Zhiyuan Liu, and Maosong Sun.
\newblock Ultrafeedback: Boosting language models with high-quality feedback.
\newblock \emph{arXiv preprint arXiv:2310.01377}, 2023.
\newblock URL \url{https://github.com/OpenBMB/UltraFeedback}.
\newblock MIT license.

\bibitem[Devlin et~al.(2019)Devlin, Chang, Lee, and Toutanova]{devlin2019bert}
Jacob Devlin, Ming-Wei Chang, Kenton Lee, and Kristina Toutanova.
\newblock Bert: Pre-training of deep bidirectional transformers for language
  understanding.
\newblock \emph{arxiv preprint arXiv:1810.04805}, 2019.

\bibitem[Dong et~al.(2023)Dong, Xiong, Goyal, Pan, Diao, Zhang, Shum, and
  Zhang]{Dong2023RAFTRR}
Hanze Dong, Wei Xiong, Deepanshu Goyal, Rui Pan, Shizhe Diao, Jipeng Zhang,
  Kashun Shum, and T.~Zhang.
\newblock {RAFT}: Reward r{A}nked {FineTuning} for generative foundation model
  alignment.
\newblock \emph{arXiv}, 2023.

\bibitem[Eisenstein et~al.(2023)Eisenstein, Nagpal, Agarwal, Beirami, D'Amour,
  Dvijotham, Fisch, Heller, Pfohl, Ramachandran, Shaw, and
  Berant]{Eisenstein2023HelpingOH}
Jacob Eisenstein, Chirag Nagpal, Alekh Agarwal, Ahmad Beirami, Alex D'Amour,
  DJ~Dvijotham, Adam Fisch, Katherine Heller, Stephen~R Pfohl, Deepak
  Ramachandran, Peter Shaw, and Jonathan Berant.
\newblock Helping or herding? {R}ward model ensembles mitigate but do not
  eliminate reward hacking.
\newblock \emph{arXiv}, 2023.

\bibitem[Ethayarajh et~al.(2024)Ethayarajh, Xu, Muennighoff, Jurafsky, and
  Kiela]{kto}
Kawin Ethayarajh, Winnie Xu, Niklas Muennighoff, Dan Jurafsky, and Douwe Kiela.
\newblock Kto: Model alignment as prospect theoretic optimization.
\newblock \emph{arxiv preprint arXiv:2402.01306}, 2024.

\bibitem[Gao et~al.(2022)Gao, Schulman, and Hilton]{Gao2022ScalingLF}
Leo Gao, John Schulman, and Jacob Hilton.
\newblock Scaling laws for reward model overoptimization.
\newblock In \emph{Proceedings of the International Conference on Machine
  Learning}, 2022.

\bibitem[Gao et~al.(2024)Gao, Chang, Zhan, Oertell, Swamy, Brantley, Joachims,
  Bagnell, Lee, and Sun]{gao2024rebel}
Zhaolin Gao, Jonathan~D. Chang, Wenhao Zhan, Owen Oertell, Gokul Swamy, Kianté
  Brantley, Thorsten Joachims, J.~Andrew Bagnell, Jason~D. Lee, and Wen Sun.
\newblock Rebel: Reinforcement learning via regressing relative rewards.
\newblock \emph{arxiv preprint arXiv:2404.16767}, 2024.

\bibitem[Glaese et~al.(2022)Glaese, McAleese, Trebacz, Aslanides, Firoiu,
  Ewalds, Rauh, Weidinger, Chadwick, Thacker, Campbell-Gillingham, Uesato,
  Huang, Comanescu, Yang, See, Dathathri, Greig, Chen, Fritz, Elias, Green,
  Mokrá, Fernando, Wu, Foley, Young, Gabriel, Isaac, Mellor, Hassabis,
  Kavukcuoglu, Hendricks, and Irving]{glaese2022improving}
Amelia Glaese, Nat McAleese, Maja Trebacz, John Aslanides, Vlad Firoiu, Timo
  Ewalds, Maribeth Rauh, Laura Weidinger, Martin Chadwick, Phoebe Thacker, Lucy
  Campbell-Gillingham, Jonathan Uesato, Po-Sen Huang, Ramona Comanescu, Fan
  Yang, Abigail See, Sumanth Dathathri, Rory Greig, Charlie Chen, Doug Fritz,
  Jaume~Sanchez Elias, Richard Green, Soňa Mokrá, Nicholas Fernando, Boxi Wu,
  Rachel Foley, Susannah Young, Iason Gabriel, William Isaac, John Mellor,
  Demis Hassabis, Koray Kavukcuoglu, Lisa~Anne Hendricks, and Geoffrey Irving.
\newblock Improving alignment of dialogue agents via targeted human judgements.
\newblock \emph{arXiv}, 2022.

\bibitem[Griffith et~al.(2013)Griffith, Subramanian, Scholz, Isbell, and
  Thomaz]{griffith2013policy}
Shane Griffith, Kaushik Subramanian, Jonathan Scholz, Charles~L Isbell, and
  Andrea~L Thomaz.
\newblock Policy shaping: Integrating human feedback with reinforcement
  learning.
\newblock In \emph{Advances in Neural Information Processing Systems}, 2013.

\bibitem[Grill et~al.(2020)Grill, Strub, Altch\'{e}, Tallec, Richemond,
  Buchatskaya, Doersch, Pires, Guo, Azar, Piot, Kavukcuoglu, Munos, and
  Valko]{BYOL}
Jean-Bastien Grill, Florian Strub, Florent Altch\'{e}, Corentin Tallec,
  Pierre~H. Richemond, Elena Buchatskaya, Carl Doersch, Bernardo~Avila Pires,
  Zhaohan~Daniel Guo, Mohammad~Gheshlaghi Azar, Bilal Piot, Koray Kavukcuoglu,
  R\'{e}mi Munos, and Michal Valko.
\newblock Bootstrap your own latent: A new approach to self-supervised
  learning.
\newblock In \emph{Advances in Neural Information Processing Systems}, 2020.

\bibitem[Guo et~al.(2024)Guo, Zhang, Liu, Liu, Khalman, Llinares, Rame,
  Mesnard, Zhao, Piot, et~al.]{guo2024direct}
Shangmin Guo, Biao Zhang, Tianlin Liu, Tianqi Liu, Misha Khalman, Felipe
  Llinares, Alexandre Rame, Thomas Mesnard, Yao Zhao, Bilal Piot, et~al.
\newblock Direct language model alignment from online ai feedback.
\newblock \emph{arXiv preprint arXiv:2402.04792}, 2024.

\bibitem[Haarnoja et~al.(2017)Haarnoja, Tang, Abbeel, and
  Levine]{haarnoja2017reinforcement}
Tuomas Haarnoja, Haoran Tang, Pieter Abbeel, and Sergey Levine.
\newblock Reinforcement learning with deep energy-based policies.
\newblock \emph{arxiv preprint arXiv:1702.08165}, 2017.

\bibitem[Hannigan et~al.(2024)Hannigan, McCarthy, and Spicer]{hannigan2024}
Timothy~R. Hannigan, Ian~P. McCarthy, and André Spicer.
\newblock Beware of botshit: How to manage the epistemic risks of generative
  chatbots.
\newblock \emph{Business Horizons}, 2024.
\newblock ISSN 0007-6813.
\newblock \doi{https://doi.org/10.1016/j.bushor.2024.03.001}.
\newblock URL
  \url{https://www.sciencedirect.com/science/article/pii/S0007681324000272}.

\bibitem[Ivison et~al.(2023)Ivison, Wang, Pyatkin, Lambert, Peters, Dasigi,
  Jang, Wadden, Smith, Beltagy, and Hajishirzi]{Ivison2023CamelsIA}
Hamish Ivison, Yizhong Wang, Valentina Pyatkin, Nathan Lambert, Matthew Peters,
  Pradeep Dasigi, Joel Jang, David Wadden, Noah~A. Smith, Iz~Beltagy, and Hanna
  Hajishirzi.
\newblock Camels in a changing climate: {E}nhancing {LM} adaptation with {Tulu}
  2.
\newblock \emph{arXiv}, 2023.

\bibitem[Jaques et~al.(2019)Jaques, Ghandeharioun, Shen, Ferguson, Lapedriza,
  Jones, Gu, and Picard]{jaques2019way}
Natasha Jaques, Asma Ghandeharioun, Judy~Hanwen Shen, Craig Ferguson, Agata
  Lapedriza, Noah Jones, Shixiang Gu, and Rosalind Picard.
\newblock Way off-policy batch deep reinforcement learning of implicit human
  preferences in dialog.
\newblock \emph{arXiv}, 2019.

\bibitem[Ji et~al.(2024)Ji, Lu, Niu, Ke, Wang, Zhu, Tang, and Huang]{exo}
Haozhe Ji, Cheng Lu, Yilin Niu, Pei Ke, Hongning Wang, Jun Zhu, Jie Tang, and
  Minlie Huang.
\newblock Towards efficient and exact optimization of language model alignment.
\newblock \emph{arxiv preprint arXiv:2402.00856}, 2024.

\bibitem[Jouppi et~al.(2023)Jouppi, Kurian, Li, Ma, Nagarajan, Nai, Patil,
  Subramanian, Swing, Towles, Young, Zhou, Zhou, and
  Patterson]{Jouppi2023TPUVA}
Norman~P. Jouppi, George Kurian, Sheng Li, Peter~C. Ma, Rahul Nagarajan, Lifeng
  Nai, Nishant Patil, Suvinay Subramanian, Andy Swing, Brian Towles, Cliff
  Young, Xiaoping Zhou, Zongwei Zhou, and David~A. Patterson.
\newblock {TPU v4}: An optically reconfigurable supercomputer for machine
  learning with hardware support for embeddings.
\newblock In \emph{Proceedings of the Annual International Symposium on
  Computer Architecture}, 2023.

\bibitem[Kahneman and Tversky(1979)]{kahnemanprospect}
D.~Kahneman and A.~Tversky.
\newblock Prospect theory: An analysis of decision under risk.
\newblock \emph{Econometrica, 47(2):263–292}, 1979.

\bibitem[Kirk et~al.(2023)Kirk, Mediratta, Nalmpantis, Luketina, Hambro,
  Grefenstette, and Raileanu]{kirk2023understanding}
Robert Kirk, Ishita Mediratta, Christoforos Nalmpantis, Jelena Luketina, Eric
  Hambro, Edward Grefenstette, and Roberta Raileanu.
\newblock Understanding the effects of rlhf on llm generalisation and
  diversity.
\newblock \emph{arXiv preprint arXiv:2310.06452}, 2023.

\bibitem[Knox and Stone(2008)]{knox2008tamer}
W~Bradley Knox and Peter Stone.
\newblock {TAMER}: Training an agent manually via evaluative reinforcement.
\newblock In \emph{Proceedings of the IEEE International Conference on
  Development and Learning}, 2008.

\bibitem[Konda and Tsitsiklis(1999)]{konda1999actor}
Vijay Konda and John Tsitsiklis.
\newblock Actor-critic algorithms.
\newblock \emph{Advances in neural information processing systems}, 12, 1999.

\bibitem[Lee et~al.(2023)Lee, Phatale, Mansoor, Lu, Mesnard, Bishop, Carbune,
  and Rastogi]{lee2023rlaif}
Harrison Lee, Samrat Phatale, Hassan Mansoor, Kellie Lu, Thomas Mesnard, Colton
  Bishop, Victor Carbune, and Abhinav Rastogi.
\newblock {RLAIF}: Scaling reinforcement learning from human feedback with {AI}
  feedback.
\newblock \emph{arXiv}, 2023.

\bibitem[Levine et~al.(2020)Levine, Kumar, Tucker, and Fu]{levine2020offline}
Sergey Levine, Aviral Kumar, George Tucker, and Justin Fu.
\newblock Offline reinforcement learning: Tutorial, review, and perspectives on
  open problems.
\newblock \emph{arXiv preprint arXiv:2005.01643}, 2020.

\bibitem[Liu et~al.(2024)Liu, Sun, and Zheng]{liu2024enhancing}
Zixuan Liu, Xiaolin Sun, and Zizhan Zheng.
\newblock Enhancing llm safety via constrained direct preference optimization,
  2024.

\bibitem[Mnih et~al.(2013)Mnih, Kavukcuoglu, Silver, Graves, Antonoglou,
  Wierstra, and Riedmiller]{Mnih13Playing}
Volodymyr Mnih, Koray Kavukcuoglu, David Silver, Alex Graves, Ioannis
  Antonoglou, Daan Wierstra, and Martin~A. Riedmiller.
\newblock Playing atari with deep reinforcement learning.
\newblock \emph{CoRR}, abs/1312.5602, 2013.

\bibitem[Mnih et~al.(2016)Mnih, Badia, Mirza, Graves, Lillicrap, Harley,
  Silver, and Kavukcuoglu]{Mnih2016asynchronous}
Volodymyr Mnih, Adri\`{a}~Puigdom\`{e}nech Badia, Mehdi Mirza, Alex Graves,
  Timothy~P. Lillicrap, Tim Harley, David Silver, and Koray Kavukcuoglu.
\newblock Asynchronous methods for deep reinforcement learning.
\newblock In \emph{Proceedings of the International Conference on Machine
  Learning}, 2016.

\bibitem[Munos et~al.(2023)Munos, Valko, Calandriello, Azar, Rowland, Guo,
  Tang, Geist, Mesnard, Michi, Selvi, Girgin, Momchev, Bachem, Mankowitz,
  Precup, and Piot]{munos2023nash}
R{\'e}mi Munos, Michal Valko, Daniele Calandriello, Mohammad~Gheshlaghi Azar,
  Mark Rowland, Daniel Guo, Yunhao Tang, Matthieu Geist, Thomas Mesnard, Andrea
  Michi, Marco Selvi, Sertan Girgin, Nikola Momchev, Olivier Bachem, Daniel~J.
  Mankowitz, Doina Precup, and Bilal Piot.
\newblock Nash learning from human feedback.
\newblock \emph{arXiv}, 2023.

\bibitem[Nakano et~al.(2021)Nakano, Hilton, Balaji, Wu, Ouyang, Kim, Hesse,
  Jain, Kosaraju, Saunders, Jiang, Cobbe, Eloundou, Krueger, Button, Knight,
  Chess, and Schulman]{nakano2021webgpt}
Reiichiro Nakano, Jacob Hilton, Suchir Balaji, Jeff Wu, Long Ouyang, Christina
  Kim, Christopher Hesse, Shantanu Jain, Vineet Kosaraju, William Saunders,
  Xu~Jiang, Karl Cobbe, Tyna Eloundou, Gretchen Krueger, Kevin Button, Matthew
  Knight, Benjamin Chess, and John Schulman.
\newblock Web{GPT}: Browser-assisted question-answering with human feedback.
\newblock \emph{arXiv}, 2021.

\bibitem[OpenAI(2022)]{chatgpt}
OpenAI.
\newblock Introducing {ChatGPT}, 2022.
\newblock URL \url{https://openai.com/blog/chatgpt}.

\bibitem[Ouyang et~al.(2022)Ouyang, Wu, Jiang, Almeida, Wainwright, Mishkin,
  Zhang, Agarwal, Slama, Ray, Schulman, Hilton, Kelton, Miller, Simens, Askell,
  Welinder, Christiano, Leike, and Lowe]{InstructGPT}
Long Ouyang, Jeff Wu, Xu~Jiang, Diogo Almeida, Carroll~L. Wainwright, Pamela
  Mishkin, Chong Zhang, Sandhini Agarwal, Katarina Slama, Alex Ray, John
  Schulman, Jacob Hilton, Fraser Kelton, Luke Miller, Maddie Simens, Amanda
  Askell, Peter Welinder, Paul Christiano, Jan Leike, and Ryan Lowe.
\newblock Training language models to follow instructions with human feedback.
\newblock \emph{arXiv}, 2022.

\bibitem[Pan et~al.(2022)Pan, Bhatia, and Steinhardt]{Pan2022TheEO}
Alexander Pan, Kush Bhatia, and Jacob Steinhardt.
\newblock The effects of reward misspecification: Mapping and mitigating
  misaligned models.
\newblock \emph{arXiv}, abs/2201.03544, 2022.

\bibitem[Pang et~al.(2022)Pang, Padmakumar, Sellam, Parikh, and
  He]{Pang2022RewardGI}
Richard~Yuanzhe Pang, Vishakh Padmakumar, Thibault Sellam, Ankur~P. Parikh, and
  He~He.
\newblock Reward gaming in conditional text generation.
\newblock In \emph{Annual Meeting of the Association for Computational
  Linguistics}, 2022.

\bibitem[Park et~al.(2024)Park, Rafailov, Ermon, and
  Finn]{park2024disentangling}
Ryan Park, Rafael Rafailov, Stefano Ermon, and Chelsea Finn.
\newblock Disentangling length from quality in direct preference optimization,
  2024.

\bibitem[Radford et~al.(2021)Radford, Kim, Hallacy, Ramesh, Goh, Agarwal,
  Sastry, Askell, Mishkin, Clark, Krueger, and Sutskever]{radford2021learning}
Alec Radford, Jong~Wook Kim, Chris Hallacy, Aditya Ramesh, Gabriel Goh,
  Sandhini Agarwal, Girish Sastry, Amanda Askell, Pamela Mishkin, Jack Clark,
  Gretchen Krueger, and Ilya Sutskever.
\newblock Learning transferable visual models from natural language
  supervision.
\newblock \emph{arxiv preprint arXiv:2103.00020}, 2021.

\bibitem[Rafailov et~al.(2023)Rafailov, Sharma, Mitchell, Ermon, Manning, and
  Finn]{rafailov2023direct}
Rafael Rafailov, Archit Sharma, Eric Mitchell, Stefano Ermon, Christopher~D
  Manning, and Chelsea Finn.
\newblock Direct preference optimization: Your language model is secretly a
  reward model.
\newblock In \emph{Advances in Neural Information Processing Systems}, 2023.

\bibitem[Raffel et~al.(2020)Raffel, Shazeer, Roberts, Lee, Narang, Matena,
  Zhou, Li, and Liu]{JMLR:v21:20-074}
Colin Raffel, Noam Shazeer, Adam Roberts, Katherine Lee, Sharan Narang, Michael
  Matena, Yanqi Zhou, Wei Li, and Peter~J. Liu.
\newblock Exploring the limits of transfer learning with a unified text-to-text
  transformer.
\newblock \emph{Journal of Machine Learning Research}, 21\penalty0
  (140):\penalty0 1--67, 2020.
\newblock URL \url{http://jmlr.org/papers/v21/20-074.html}.

\bibitem[Ram{\'e} et~al.(2024)Ram{\'e}, Vieillard, Hussenot, Dadashi, Cideron,
  Bachem, and Ferret]{Rame2024WARMOT}
Alexandre Ram{\'e}, Nino Vieillard, L{\'e}onard Hussenot, Robert Dadashi,
  Geoffrey Cideron, Olivier Bachem, and Johan Ferret.
\newblock {WARM}: On the benefits of weight averaged reward models.
\newblock \emph{arXiv}, 2024.

\bibitem[Richemond and Maginnis(2017)]{richemond2017short}
Pierre~H. Richemond and Brendan Maginnis.
\newblock A short variational proof of equivalence between policy gradients and
  soft q learning.
\newblock \emph{arxiv preprint arXiv:1712.08650}, 2017.

\bibitem[Roberts et~al.(2022)Roberts, Chung, Levskaya, Mishra, Bradbury, Andor,
  Narang, Lester, Gaffney, Mohiuddin, Hawthorne, Lewkowycz, Salcianu, van Zee,
  Austin, Goodman, Soares, Hu, Tsvyashchenko, Chowdhery, Bastings, Bulian,
  Garcia, Ni, Chen, Kenealy, Clark, Lee, Garrette, Lee-Thorp, Raffel, Shazeer,
  Ritter, Bosma, Passos, Maitin-Shepard, Fiedel, Omernick, Saeta, Sepassi,
  Spiridonov, Newlan, and Gesmundo]{roberts2022t5x}
Adam Roberts, Hyung~Won Chung, Anselm Levskaya, Gaurav Mishra, James Bradbury,
  Daniel Andor, Sharan Narang, Brian Lester, Colin Gaffney, Afroz Mohiuddin,
  Curtis Hawthorne, Aitor Lewkowycz, Alex Salcianu, Marc van Zee, Jacob Austin,
  Sebastian Goodman, Livio~Baldini Soares, Haitang Hu, Sasha Tsvyashchenko,
  Aakanksha Chowdhery, Jasmijn Bastings, Jannis Bulian, Xavier Garcia, Jianmo
  Ni, Andrew Chen, Kathleen Kenealy, Jonathan~H. Clark, Stephan Lee, Dan
  Garrette, James Lee-Thorp, Colin Raffel, Noam Shazeer, Marvin Ritter, Maarten
  Bosma, Alexandre Passos, Jeremy Maitin-Shepard, Noah Fiedel, Mark Omernick,
  Brennan Saeta, Ryan Sepassi, Alexander Spiridonov, Joshua Newlan, and Andrea
  Gesmundo.
\newblock Scaling up models and data with $\texttt{t5x}$ and $\texttt{seqio}$.
\newblock \emph{arXiv}, 2022.
\newblock URL \url{https://github.com/google-research/t5x}.
\newblock Apache-2.0 license.

\bibitem[Rosset et~al.(2024)Rosset, Cheng, Mitra, Santacroce, Awadallah, and
  Xie]{rosset2024direct}
Corby Rosset, Ching-An Cheng, Arindam Mitra, Michael Santacroce, Ahmed
  Awadallah, and Tengyang Xie.
\newblock Direct nash optimization: Teaching language models to self-improve
  with general preferences, 2024.

\bibitem[Saunders et~al.(2022)Saunders, Yeh, Wu, Bills, Ouyang, Ward, and
  Leike]{saunders2022self}
William Saunders, Catherine Yeh, Jeff Wu, Steven Bills, Long Ouyang, Jonathan
  Ward, and Jan Leike.
\newblock Self-critiquing models for assisting human evaluators.
\newblock \emph{arXiv preprint arXiv:2206.05802}, 2022.

\bibitem[Schulman et~al.(2017)Schulman, Wolski, Dhariwal, Radford, and
  Klimov]{PPO}
John Schulman, Filip Wolski, Prafulla Dhariwal, Alec Radford, and Oleg Klimov.
\newblock Proximal policy optimization algorithms.
\newblock \emph{arXiv}, 2017.

\bibitem[Schulman et~al.(2018)Schulman, Chen, and
  Abbeel]{schulman2018equivalence}
John Schulman, Xi~Chen, and Pieter Abbeel.
\newblock Equivalence between policy gradients and soft q-learning.
\newblock \emph{arxiv preprint arXiv:1704.06440}, 2018.

\bibitem[Shazeer and Stern(2018)]{Shazeer2018AdafactorAL}
Noam~M. Shazeer and Mitchell Stern.
\newblock Adafactor: Adaptive learning rates with sublinear memory cost.
\newblock \emph{arXiv}, 2018.

\bibitem[Shin et~al.(2023)Shin, Dragan, and Brown]{Shin2023BenchmarksAA}
Daniel Shin, Anca~D. Dragan, and Daniel~S. Brown.
\newblock Benchmarks and algorithms for offline preference-based reward
  learning.
\newblock \emph{arXiv}, 2023.

\bibitem[Singhal et~al.(2023)Singhal, Goyal, Xu, and Durrett]{Singhal2023ALW}
Prasann Singhal, Tanya Goyal, Jiacheng Xu, and Greg Durrett.
\newblock A long way to go: Investigating length correlations in rlhf.
\newblock \emph{ArXiv}, abs/2310.03716, 2023.

\bibitem[Skalse et~al.(2022)Skalse, Howe, Krasheninnikov, and
  Krueger]{Skalse2022DefiningAC}
Joar Skalse, Nikolaus H.~R. Howe, Dmitrii Krasheninnikov, and David Krueger.
\newblock Defining and characterizing reward gaming.
\newblock In \emph{Neural Information Processing Systems}, 2022.

\bibitem[Sun et~al.(2023)Sun, Shen, Cao, Liu, Li, Shen, Gan, Gui, Wang, Yang,
  Keutzer, and Darrell]{sun2023aligning}
Zhiqing Sun, Sheng Shen, Shengcao Cao, Haotian Liu, Chunyuan Li, Yikang Shen,
  Chuang Gan, Liang-Yan Gui, Yu-Xiong Wang, Yiming Yang, Kurt Keutzer, and
  Trevor Darrell.
\newblock Aligning large multimodal models with factually augmented rlhf, 2023.

\bibitem[Sutton et~al.(2000)Sutton, McAllester, Singh, and
  Mansour]{Sutton00policy}
R.~S. Sutton, D.~McAllester, S.~Singh, and Y.~Mansour.
\newblock Policy gradient methods for reinforcement learning with function
  approximation.
\newblock In \emph{Advances in Neural Information Processing Systems 12}, pages
  1057--1063. {MIT} Press, 2000.

\bibitem[Swamy et~al.(2024)Swamy, Dann, Kidambi, Wu, and Agarwal]{Swamy2024AMA}
Gokul Swamy, Christoph Dann, Rahul Kidambi, Zhiwei~Steven Wu, and Alekh
  Agarwal.
\newblock A minimaximalist approach to reinforcement learning from human
  feedback.
\newblock \emph{arXiv}, 2024.

\bibitem[Tang et~al.(2024)Tang, Guo, Zheng, Calandriello, Munos, Rowland,
  Richemond, Valko, Pires, and Piot]{tang2024generalized}
Yunhao Tang, Zhaohan~Daniel Guo, Zeyu Zheng, Daniele Calandriello, R{\'e}mi
  Munos, Mark Rowland, Pierre~Harvey Richemond, Michal Valko,
  Bernardo~{\'A}vila Pires, and Bilal Piot.
\newblock Generalized preference optimization: A unified approach to offline
  alignment.
\newblock \emph{arXiv preprint arXiv:2402.05749}, 2024.

\bibitem[Touvron et~al.(2023)Touvron, Martin, Stone, Albert, Almahairi, Babaei,
  Bashlykov, Batra, Bhargava, Bhosale, Bikel, Blecher, Ferrer, Chen, Cucurull,
  Esiobu, Fernandes, Fu, Fu, Fuller, Gao, Goswami, Goyal, Hartshorn, Hosseini,
  Hou, Inan, Kardas, Kerkez, Khabsa, Kloumann, Korenev, Koura, Lachaux, Lavril,
  Lee, Liskovich, Lu, Mao, Martinet, Mihaylov, Mishra, Molybog, Nie, Poulton,
  Reizenstein, Rungta, Saladi, Schelten, Silva, Smith, Subramanian, Tan, Tang,
  Taylor, Williams, Kuan, Xu, Yan, Zarov, Zhang, Fan, Kambadur, Narang,
  Rodriguez, Stojnic, Edunov, and Scialom]{Touvron2023Llama2O}
Hugo Touvron, Louis Martin, Kevin~R. Stone, Peter Albert, Amjad Almahairi,
  Yasmine Babaei, Nikolay Bashlykov, Soumya Batra, Prajjwal Bhargava, Shruti
  Bhosale, Daniel~M. Bikel, Lukas Blecher, Cristian~Cant{\'o}n Ferrer, Moya
  Chen, Guillem Cucurull, David Esiobu, Jude Fernandes, Jeremy Fu, Wenyin Fu,
  Brian Fuller, Cynthia Gao, Vedanuj Goswami, Naman Goyal, Anthony~S.
  Hartshorn, Saghar Hosseini, Rui Hou, Hakan Inan, Marcin Kardas, Viktor
  Kerkez, Madian Khabsa, Isabel~M. Kloumann, A.~V. Korenev, Punit~Singh Koura,
  Marie-Anne Lachaux, Thibaut Lavril, Jenya Lee, Diana Liskovich, Yinghai Lu,
  Yuning Mao, Xavier Martinet, Todor Mihaylov, Pushkar Mishra, Igor Molybog,
  Yixin Nie, Andrew Poulton, Jeremy Reizenstein, Rashi Rungta, Kalyan Saladi,
  Alan Schelten, Ruan Silva, Eric~Michael Smith, R.~Subramanian, Xia Tan, Binh
  Tang, Ross Taylor, Adina Williams, Jian~Xiang Kuan, Puxin Xu, Zhengxu Yan,
  Iliyan Zarov, Yuchen Zhang, Angela Fan, Melanie Kambadur, Sharan Narang,
  Aurelien Rodriguez, Robert Stojnic, Sergey Edunov, and Thomas Scialom.
\newblock Llama 2: Open foundation and fine-tuned chat models.
\newblock \emph{arXiv}, 2023.

\bibitem[Tunstall et~al.(2023)Tunstall, Beeching, Lambert, Rajani, Rasul,
  Belkada, Huang, von Werra, Fourrier, Habib, Sarrazin, Sanseviero, Rush, and
  Wolf]{tunstall2023zephyr}
Lewis Tunstall, Edward Beeching, Nathan Lambert, Nazneen Rajani, Kashif Rasul,
  Younes Belkada, Shengyi Huang, Leandro von Werra, Clémentine Fourrier,
  Nathan Habib, Nathan Sarrazin, Omar Sanseviero, Alexander~M. Rush, and Thomas
  Wolf.
\newblock Zephyr: Direct distillation of {LM} alignment.
\newblock \emph{arXiv}, 2023.

\bibitem[Wang et~al.(2023)Wang, Jiang, Yang, Liu, and Chen]{Wang2023BeyondRK}
Chaoqi Wang, Yibo Jiang, Chenghao Yang, Han Liu, and Yuxin Chen.
\newblock Beyond reverse {KL}: Generalizing direct preference optimization with
  diverse divergence constraints.
\newblock \emph{arXiv}, 2023.

\bibitem[Warnell et~al.(2018)Warnell, Waytowich, Lawhern, and
  Stone]{warnell2018deep}
Garrett Warnell, Nicholas Waytowich, Vernon Lawhern, and Peter Stone.
\newblock Deep {TAMER}: Interactive agent shaping in high-dimensional state
  spaces.
\newblock In \emph{Proceedings of the AAAI Conference on Artificial
  Intelligence}, 2018.

\bibitem[Wortsman et~al.(2022)Wortsman, Ilharco, Gadre, Roelofs, Gontijo-Lopes,
  Morcos, Namkoong, Farhadi, Carmon, Kornblith, and
  Schmidt]{Wortsman2022ModelSA}
Mitchell Wortsman, Gabriel Ilharco, Samir~Yitzhak Gadre, Rebecca Roelofs,
  Raphael Gontijo-Lopes, Ari~S. Morcos, Hongseok Namkoong, Ali Farhadi, Yair
  Carmon, Simon Kornblith, and Ludwig Schmidt.
\newblock Model soups: averaging weights of multiple fine-tuned models improves
  accuracy without increasing inference time.
\newblock In \emph{Proceedings of the International Conference on Machine
  Learning}, 2022.

\bibitem[Wu et~al.(2023)Wu, Hu, Shi, Dziri, Suhr, Ammanabrolu, Smith,
  Ostendorf, and Hajishirzi]{wu2023finegrained}
Zeqiu Wu, Yushi Hu, Weijia Shi, Nouha Dziri, Alane Suhr, Prithviraj
  Ammanabrolu, Noah~A. Smith, Mari Ostendorf, and Hannaneh Hajishirzi.
\newblock Fine-grained human feedback gives better rewards for language model
  training, 2023.

\bibitem[Xu et~al.(2024)Xu, Fu, Gao, Ye, Liu, Mei, Wang, Yu, and Wu]{xu2024dpo}
Shusheng Xu, Wei Fu, Jiaxuan Gao, Wenjie Ye, Weilin Liu, Zhiyu Mei, Guangju
  Wang, Chao Yu, and Yi~Wu.
\newblock Is dpo superior to ppo for llm alignment? a comprehensive study.
\newblock \emph{arxiv preprint arXiv:2404.10719}, 2024.

\bibitem[Yuan et~al.(2024{\natexlab{a}})Yuan, Pang, Cho, Sukhbaatar, Xu, and
  Weston]{yuan2024selfrewarding}
Weizhe Yuan, Richard~Yuanzhe Pang, Kyunghyun Cho, Sainbayar Sukhbaatar, Jing
  Xu, and Jason Weston.
\newblock Self-rewarding language models, 2024{\natexlab{a}}.

\bibitem[Yuan et~al.(2024{\natexlab{b}})Yuan, Hao, Ma, Dong, Liang, Liu, Feng,
  Zhao, and Zheng]{yuan2024unirlhf}
Yifu Yuan, Jianye Hao, Yi~Ma, Zibin Dong, Hebin Liang, Jinyi Liu, Zhixin Feng,
  Kai Zhao, and Yan Zheng.
\newblock Uni-rlhf: Universal platform and benchmark suite for reinforcement
  learning with diverse human feedback, 2024{\natexlab{b}}.

\bibitem[Yuan et~al.(2023)Yuan, Yuan, Tan, Wang, Huang, and
  Huang]{Yuan2023RRHFRR}
Zheng Yuan, Hongyi Yuan, Chuanqi Tan, Wei Wang, Songfang Huang, and Feiran
  Huang.
\newblock Rrhf: Rank responses to align language models with human feedback
  without tears.
\newblock \emph{arXiv}, abs/2304.05302, 2023.

\bibitem[Zeng et~al.(2024)Zeng, Liu, Ma, Yang, Zhang, and
  Wang]{zeng2024tokenlevel}
Yongcheng Zeng, Guoqing Liu, Weiyu Ma, Ning Yang, Haifeng Zhang, and Jun Wang.
\newblock Token-level direct preference optimization, 2024.

\bibitem[Zhang et~al.(2024)Zhang, Chen, Chen, Shen, Sun, and
  Gan]{zhang2024improving}
Shun Zhang, Zhenfang Chen, Sunli Chen, Yikang Shen, Zhiqing Sun, and Chuang
  Gan.
\newblock Improving reinforcement learning from human feedback with efficient
  reward model ensemble, 2024.

\bibitem[Zhao et~al.(2023{\natexlab{a}})Zhao, Joshi, Liu, Khalman, Saleh, and
  Liu]{zhao2023slichf}
Yao Zhao, Rishabh Joshi, Tianqi Liu, Misha Khalman, Mohammad Saleh, and
  Peter~J. Liu.
\newblock {SLiC-HF}: Sequence likelihood calibration with human feedback.
\newblock \emph{arXiv}, 2023{\natexlab{a}}.

\bibitem[Zhao et~al.(2023{\natexlab{b}})Zhao, Wang, Ouyang, Dong, Wang, and
  He]{zhao2023beyond}
Zhiyuan Zhao, Bin Wang, Linke Ouyang, Xiaoyi Dong, Jiaqi Wang, and Conghui He.
\newblock Beyond hallucinations: Enhancing lvlms through hallucination-aware
  direct preference optimization.
\newblock \emph{arXiv preprint arXiv:2311.16839}, 2023{\natexlab{b}}.

\bibitem[Zheng et~al.(2023)Zheng, Chiang, Sheng, Zhuang, Wu, Zhuang, Lin, Li,
  Li, Xing, Zhang, Gonzalez, and Stoica]{zheng2023judging}
Lianmin Zheng, Wei-Lin Chiang, Ying Sheng, Siyuan Zhuang, Zhanghao Wu, Yonghao
  Zhuang, Zi~Lin, Zhuohan Li, Dacheng Li, Eric~P. Xing, Hao Zhang, Joseph~E.
  Gonzalez, and Ion Stoica.
\newblock Judging llm-as-a-judge with mt-bench and chatbot arena, 2023.

\bibitem[Zhuang and Hadfield-Menell(2020)]{Zhuang2021ConsequencesOM}
Simon Zhuang and Dylan Hadfield-Menell.
\newblock Consequences of misaligned {AI}.
\newblock In \emph{Advances in Neural Information Processing Systems}, 2020.

\bibitem[Ziebart et~al.(2008)Ziebart, Maas, Bagnell, Dey,
  et~al.]{ziebart2008maximum}
Brian~D Ziebart, Andrew~L Maas, J~Andrew Bagnell, Anind~K Dey, et~al.
\newblock Maximum entropy inverse reinforcement learning.
\newblock In \emph{Aaai}, volume~8, pages 1433--1438. Chicago, IL, USA, 2008.

\bibitem[Ziegler et~al.(2020)Ziegler, Stiennon, Wu, Brown, Radford, Amodei,
  Christiano, and Irving]{ziegler2020finetuning}
Daniel~M. Ziegler, Nisan Stiennon, Jeffrey Wu, Tom~B. Brown, Alec Radford,
  Dario Amodei, Paul Christiano, and Geoffrey Irving.
\newblock Fine-tuning language models from human preferences.
\newblock \emph{arXiv preprint arXiv:1909.08593}, 2020.

\end{thebibliography}
\bibliographystyle{plainnat}


\newpage
\appendix
\section*{\centering APPENDICES}
\section{Mathematical Proofs}\label{sec:appendix-proofs}

We recall Proposition \ref{prop:approx-error} and proceed to prove it here:

\begin{proposition*}
Consider an approximation $V(x)$ of the value function $V^\pi=\arg\min_V \L(\pi,V)$: 
$$V^\pi(x)= \E_{y\sim \mu(\cdot|x)} \left[ r(x,y) - \tau \log\frac{\pi(y|x)}{\piref(y|x)} \right]$$ 
and for this function kept fixed, let us optimise the loss $\pi\mapsto \L(\pi, V)$ with respect to the policy only. Let $\pi_V = \arg\min_\pi \L(\pi, V)$. Then we have that $\pi_V$ satisfies the following equation:
\begin{equation}\label{eq:first.claim2}
    \pi_V(y|x) \propto \piref(y|x) e^{\frac 1\tau \left[r(x,y) - \frac{\pi_V(y|x)}{\mu(y|x)}\left(V^{\pi_V}(x) - V(x) \right) \right]}.
\end{equation}

In addition we have that for all $x,y$,
\beq
\left|\log\frac{\pi_V(y|x)}{\pi^*(y|x)}\right|
\leq \frac 2\tau  \max_y \left| \left( V^{\pi_V}(x) - V(x)\right) \left( 1 - \frac{\pi_V(y|x)}{\mu(y|x)}\right) \right|.
\eeq
\end{proposition*}

\begin{proof} We proceed to prove this here.
Consider the Lagrangian
${\cal G}(\pi) = {\cal L}(\pi,V) + \sum_x\rho(x)\lambda_x(\sum_y \pi(y|x)-1)$ where $\lambda_x$ are the Lagrange multipliers corresponding to the constraint $\sum_y \pi(y|x)=1$. Notice there is no need to impose such constraint outside of the support of $\rho$. The derivative of ${\cal G}(\pi)$ w.r.t.~a variable $\pi(y|x)$ is
\begin{eqnarray*}
\partial_{\pi(y|x)}{\cal G}(\pi) &=&  \E_{x'\sim\rho, y'\sim\mu(\cdot|x')} \left[ \left(   r(x',y') - \tau \log\frac{\pi(y'|x')}{\piref(y'|x')} - V(x') \right) \partial_{\pi(y|x)}\left(-\tau \log\frac{\pi(y'|x')}{\piref(y'|x')}\right) \right] \\
& & + \partial_{\pi(y|x)}\left( \sum_{x'}\rho(x')\lambda_{x'}\sum_{y'}\pi(y'|x')\right)\\
&=&  - \rho(x)\left[ \mu(y|x) \left(  r(x,y) - \tau \log\frac{\pi(y|x)}{\piref(y|x)} - V(x) \right) \frac{\tau}{\pi(y|x)}  - \lambda_x\right].
\end{eqnarray*}
Setting the optimality conditions $\partial_{\pi(y|x)}{\cal G}(\pi)=0$ for all $x$ and $y$, we deduce that for any $x$ in the support of $\rho$, we have, for any $y$,
\begin{eqnarray*}
\lambda_x &=& \tau \frac{\mu(y|x)}{\pi(y|x)} \left(  r(x,y) - \tau \log\frac{\pi(y|x)}{\piref(y|x)} - V(x) \right) \\
&=& \tau \sum_{y}\pi(y|x) \frac{\mu(y|x)}{\pi(y|x)} \left(  r(x,y) - \tau \log\frac{\pi(y|x)}{\piref(y|x)} - V(x) \right) \\
&=& \tau \sum_{y}\mu(y|x) \left(  r(x,y) - \tau \log\frac{\pi(y|x)}{\piref(y|x)} - V(x) \right) \\
&=& \tau \left(V^\pi(x)- V(x)\right).
\end{eqnarray*}
Thus 
$$\frac{\mu(y|x)}{\pi(y|x)} \left(  r(x,y) -\tau  \log\frac{\pi(y|x)}{\piref(y|x)} - V(x) \right) = V^\pi(x)- V(x),$$
or, equivalently:
\beq\label{eq:1}
\tau \log\frac{\pi(y|x)}{\piref(y|x)} = r(x,y) - V(x) - \frac{\pi(y|x)}{\mu(y|x)}\left(V^\pi(x)- V(x)\right),
\eeq
from which we deduce that the optimal policy 
satisfies
\begin{align}
\pi(y|x) &=\piref(y|x) e^{\frac 1\tau \left[ r(x,y) - \frac{\pi(y|x)}{\mu(y|x)}\left(V^{\pi}(x) - V(x) \right) -V(x) \right]} \\ &\propto \piref(y|x) e^{\frac 1\tau \left[ r(x,y) - \frac{\pi(y|x)}{\mu(y|x)}\left(V^{\pi}(x) - V(x) \right) \right]}
\end{align}
which proves the first claim \eqref{eq:first.claim2}.

Now from the property of the optimal policy $\pi^*$, and from \eqref{eq:1}, we deduce that
\beq\label{eq:2}
\tau \log\frac{\pi(y|x)}{\pi^*(y|x)} =  V^*(x) - V(x) - \frac{\pi(y|x)}{\mu(y|x)}\left(V^\pi(x)- V(x)\right).
\eeq
Taking the expectation w.r.t.~$\mu(\cdot|x)$ on both sides, we get
$$
\tau \sum_y\mu(y|x) \log\frac{\pi(y|x)}{\pi^*(y|x)} =  V^*(x) - V^\pi(x).
$$
Plugging back $V^*(x)$ into \eqref{eq:2} we have for any $x,y$,
\beqan
\tau \log\frac{\pi(y|x)}{\pi^*(y|x)} &=&  \tau \sum_{y'}\mu(y'|x) \log\frac{\pi(y'|x)}{\pi^*(y'|x)} + V^\pi(x) - V(x) - \frac{\pi(y|x)}{\mu(y|x)}\left(V^\pi(x)- V(x)\right) \\
&=&  \tau \sum_{y'}\mu(y'|x) \log\frac{\pi(y'|x)}{\pi^*(y'|x)} + \left( V^\pi(x) - V(x)\right) \left( 1 - \frac{\pi(y|x)}{\mu(y|x)}\right).
\eeqan 
Let us write $\epsilon(x) \eqdef \frac 1\tau \max_y \left| \left( V^\pi(x) - V(x)\right) \left( 1 - \frac{\pi(y|x)}{\mu(y|x)}\right) \right|$ and $\rho(y|x)\eqdef \log\frac{\pi(y|x)}{\pi^*(y|x)}$. We have that
\beqan
1 = \sum_y \pi(y|x) &=& \sum_y \pi^*(y|x) e^{\rho(y|x)} \\
&\leq& \sum_y \pi^*(y|x) e^{\sum_{y'}\mu(y'|x)\rho(y'|x) + \epsilon(x)}\\
&=& e^{\sum_{y'}\mu(y'|x)\rho(y'|x) + \epsilon(x)},
\eeqan 
thus $\sum_{y'}\mu(y'|x)\rho(y'|x) \geq -\epsilon(x)$. Similarly we can upper bound $\sum_{y'}\mu(y'|x)\rho(y'|x)$ by $\epsilon(x)$.

We deduce that for all $y$
\beqan 
|\rho(y|x)|&\leq& \left|\sum_{y'}\mu(y'|x)\rho(y'|x)\right| + \epsilon(x)\\
&\leq& 2 \epsilon(x),
\eeqan
which concludes the proof of the second claim.
\end{proof}


\section{A policy-only algorithmic variant}\label{sec:stipo-novalue}

To reduce the computational burden incurred by \texttt{DRO-V}, where we use typically use two different neural networks for $\pi$ and $V$, we can also define a loss function over policy $\pi$ alone. To this end, and as per Section \ref{sec:approx-error}, we can define $V^\pi=\arg\min_V {\cal L}(\pi,V)$ as the best value function given a fixed policy $\pi$. Writing down the optimality condition we recall that, as per Proposition \ref{prop:approx-error}:
$$V^\pi(x)=\E_{y\sim\mu} \left[r(x,y) - \tau \log\frac{\pi(y|x)}{\piref(y|x)}\right].$$

From this we deduce and define the loss over policy only, $\L(\pi)$:
\begin{eqnarray*}
\L(\pi)&\eqdef& \L(\pi,V^\pi)\\
&=& \frac 12 \E_{x\sim\rho, y\sim\mu(\cdot|x)} \left[ \left(
r(x,y) -\tau  \log\frac{\pi(y|x)}{\piref(y|x)} - 
\E_{y'\sim\mu(\cdot|x)} \left[ r(x,y') - \tau \log\frac{\pi(y'|x)}{\piref(y'|x)}\right]
\right)^2\right]\\
&=& \frac 12 \E_{x\sim\rho}\left[\mbox{Var}_{y\sim\mu(\cdot|x)} \left(
r(x,y) -\tau  \log\frac{\pi(y|x)}{\piref(y|x)} \right) \right].
\end{eqnarray*}

We notice that operationalising this loss requires evaluating an empirical estimate for the variance of a residual term, $\mbox{Var}_{y\sim\mu(\cdot|x)}$, over completions $y$ given a prompt $x$. Since the variance cannot be estimated from a single sample $y\sim \pi(\cdot|x)$, we cannot practically implement the loss function in an unbiased way. However, if we allow for two samples estimation, i.e., $y_1,y_2\sim \pi(\cdot|x)$, this will deviate from the single trajectory setting but closely connect to the above loss to the pairwise preference loss. Most notably, we can rewrite
\begin{align*}
    \L(\pi) = \frac 12  \E_{y_i\sim \pi(\cdot|x),x\sim\rho}\left[\left(
r(x,y_2) - r(x,y_1) + \tau  \log\frac{\pi(y_1|x)}{\piref(y_1|x)} - \tau \log\frac{\pi(y_2|x)}{\piref(y_2|x)}  \right)^2 \right].
\end{align*}
Note that the reward difference $r(x,y_2)-r(x,y_1)$ can be understood as a form of preference of $y_2$ over $y_1$. Indeed, when replacing the difference by a preference indicator, the above loss simply reduces to the IPO loss \citep{azar2023general}.


However, as we are placed in the single trajectory setting rather than the preference setting (where a $2$-sample estimate of the variance would be available for each prompt), we resort instead to a tractable approximation for $\mathcal{L}(\pi)$. We estimate the variance term $\mbox{Var}_{y\sim\mu(\cdot|x)}(\cdot)$ as variance on each mini-batch $B$ instead, where each element has a different $x$. Thus we replace this term with empirical estimate $\mbox{Var}_{y\sim\B}$, involving computation of the empirical variance of log-policy ratios along the batch axis. This gives us a tractable, approximate policy-only variant for DRO, where the loss is parameterised by $\theta$ only:
\begin{equation}\label{eq:stipov-policyonly}
\L_{\texttt{DRO-P}}(\theta) \eqdef \frac 12 \E_{x\sim\rho}\left[\mbox{Var}_{y\sim\B } \left(
r(x,y) -\tau  \log\frac{\pi_\theta(y|x)}{\piref(y|x)} \right) \right]
\end{equation}

We present the associated Algorithm \ref{alg:algo-novalue}, that we will call \texttt{DRO-P}, below.

\begin{algorithm}[h!]
   \caption{\texttt{\algo}~\texttt{P}olicy Only Optimisation : \texttt{DRO-P} }
   \label{alg:algo-novalue}
\begin{algorithmic}
   \STATE {\bfseries Inputs:} A single-trajectory dataset: $(x_i, y_i, r_i=r(x_i, y_i))_{i=1}^N$, a parameterised policy: $\pi_\theta$, a reference policy: $\piref$, a regularisation scalar $\tau$, a number of total steps $K$, a batch size $B$ and an optimiser.  

   \FOR{$k=1$ {\bfseries to} $K$}
   \STATE Sample uniformly a batch: $(x_i, y_i, r_i)_{i=1}^B$
   \STATE Compute gradient updates $\nabla_\theta {\L}_{\texttt{DRO-P}}(\theta)$ for the sampled loss defined by \eqref{eq:stipov-policyonly}:
   
   $$\nabla_\theta {\L}_{\texttt{DRO-P}}(\theta) = -\frac{1}{B} \sum_{i=1}^B \Big( \nabla_\theta \log \pi_\theta(y_i|x_i) r(x_i,y_i) - \tau \nabla_\theta \mbox{Var}_{y\sim\B } \Big( \log\frac{\pi_\theta(y_i|x_i)}{\piref(y_i|x_i)} \Big) \Big)$$
 
   \STATE Update the policy parameters: $\theta \gets \texttt{UpdateOptimiser}(\theta, \nabla_\theta {\L}_{\texttt{DRO-P}}(\theta))$
  
   \ENDFOR
   \STATE {\bfseries Outputs:} $\pi_\theta$
\end{algorithmic}
\end{algorithm}

\texttt{DRO-P} appears similar to a standard policy gradients RL algorithm, featuring an additional variance regulariser computed on the batch. We found that \texttt{DRO-P} in practice does not collapse and produces non-trivial resulting policies, that can in fact beat the SFT policy in side-by-side comparison on our evaluation setting (T5-encoder on UltraFeedback task, as in Experiments section \ref{sec:experiments}). Nonetheless, this algorithm, while being simpler and more parsimonious than \texttt{\algo-V}, is also less competitive. This speaks to the necessity to use both a policy and a value function - in a theoretically sound and principled way - in order to achieve the best empirical results, and is why we emphasised \texttt{DRO-V} in our presentation. Finally, this result also shows that DRO is a \emph{framework} that can be instantiated in a variety of algorithmic ways.

\newpage

\section{Additional Empirical Results}\label{sec:appendix-experiments}

\subsection{Parameter sharing ablation}

Here we provide the full results of a pairwise side-by-side comparison involving all architectural modifications, namely either using a single neural network or two networks for $\pi$ and $V$, and whether to use a state-independent function $V$ or not (a single or multiple value number per batch). Results are presented below Figure \ref{fig:appendix-ablation-paramshare}, complementing those from Figure \ref{fig:ablation-paramshare} in the main text, which only presented comparison against the SFT.

\begin{figure}[!ht]
\centering
\begin{minipage}{0.99\textwidth} 
\centering
\footnotesize
\begin{tabular}{c|ccccc}
\toprule
Parameter Sharing Variant & \makecell[c]{Double Net\\Single Value} & \makecell[c]{Single Net\\Single Value} & \makecell[c]{Single Net\\Multiple Values} & \makecell[c]{Double Net\\Multiple Values} & \makecell[c]{SFT} \\
\midrule
Double net single value & $50.0\%$ & $37.1\%$ & $28.9\%$ & $54.9\%$ & $27.9\%$ \\
Single net single value & $62.9\%$ & $50.0\%$ & $45.2\%$ & $70.4\%$ & $42.4\%$ \\
Single net multiple value & $71.1\%$ & $54.8\%$ & $50.0\%$ & $74.7\%$ & $44.5\%$ \\
Double net multiple value & $45.1\%$ & $29.6\%$ & $25.3\%$ & $50.0\%$ & $23.4\%$ \\
SFT & $72.1\%$ & $57.6\%$ & $55.5\%$ & $76.6\%$ & $50.0\%$ \\
\bottomrule
\end{tabular}

\caption{Parameter sharing variants. \texttt{DRO-V} variants' side-by-side winrate.}
\label{fig:appendix-ablation-paramshare}
\end{minipage}
\end{figure}

\section{Discussion on the cancellation of partition function}
\label{sec:appendix-technicalities}
We provide additional discussion on the cancellation of partition function in the pairwise preference optimisation algorithms. Starting with the regularised policy optimisation setting in \eqref{eq:setup}, we can write
\begin{align*}
   \tau \log \frac{\pi^\ast(y|x)}{\piref (y|x)} - V(x) = r(x,y).
\end{align*}
The derivation of pairwise preference optimisation algorithms typically casts the offline policy optimisation problem into a reward modeling problem. That is, one can parameterise a reward function $r$ using the equation above, and plug into a reward modeling loss. 

Many of the pairwise preference optimisation algorithms make use of the paired data $(x,y_l,y_w)$ where given a prompt $x$, the completion $y_w$ is preferred over $y_l$ by human raters. The loss function takes the form of taking the reward difference between the two completions, resulting in
$$r(x,y_w)-r(x,y_l) =\tau \log \frac{\pi^\ast(y_w|x)}{\piref (y_w|x)}-\tau \log \frac{\pi^\ast(y_l|x)}{\piref (y_l|x)} $$
which \emph{cancels} out the partition function or the value function $V(x)$. Hence, in the pairwise preference optimisation setting, the partition function $V$ is implicit in the derivation and we do not need to parameterise or learn it explicitly. This is unlike the single trajectory setting, where the partition function does not cancel out and hence needs to be learned explicitly.

\section{Qualitative study of prompt completions}
\label{sec:appendix-qualitative}

Here we show some example side-by-side prompt completions for our best \texttt{DRO-V} and KTO checkpoints (as evaluated in Section \ref{sec:experiments}), using \emph{T5-XL} text encoders. We show those under the format 'Prompt - Response 1 ( \texttt{DRO-V} ) - Response 2 (KTO)'. Overall, \texttt{DRO-V} responses appear slightly more helpful and focused, and are generally preferred by the Palm2 LLM judge. We found that it's often a combination of helpfulness and concision, rather than just one of these dimensions, that tilts the decision one way or another. In order to be as representative as possible, we try to illustrate both some clear preference cases, as well as others that the reader might find more subtle.

\newpage

\begin{shaded}
\texttt{\underline{Prompt}: \small List the components of an Ethernet cable.}

\texttt{Response 1: \small ['Ethernet cables are made up of several components, including conductors (conductors), wires (coaxial cables), and connectors. The components of an Ethernet cable are designed to
transmit data and signals from a computer over a network. The following are the components of an Ethernet cable: 1. Conductors: Conductors are the core of the Ethernet cable. They
are made up of copper wires in pairs, joined together by a metallic strip. The wires in the Ethernet cable form a network. 2. Wires: These wires are the insulators, or inner
shield, of the Ethernet cable. They form the outer layer of the cable, preventing the signals from being transmitted and prevent electromagnetic interference from the network. 3.
Connectors: Connectors are the wires that carry data and signals from the Ethernet cable to the network. They are made up of copper and non-conductive metal wires. The types of
Ethernet cables available include Cat 5, Cat 7, and Category 5.']}

\texttt{Response 2: \small ['Ethernet cables are made up of several components, including conductors (conductors), wires (modems), and casings to shield and protect them from the elements. They use twisted
pairs to join them together, which are insulated with copper. This creates a sturdy and flexible cable that has a low impedance, making it ideal for network connections. Ethernet
cables can also include other components, such as switches and hubs, but such items are uncommon in Ethernet-based networks. For more information on Ethernet, refer to "The History
of Computer Networks."']}
\end{shaded}

\begin{shaded}
\texttt{\underline{Prompt}: \small Q: You are given a sentence and a question in the input. The information provided in the sentence should be enough to answer the question. You're expected to write the correct
answer. Do not use any facts other than those provided in the sentence. Sentence: They were Dig the dog, Mittens the rooster, and Bub the duck. Question: Who made the noise that
caused the sock monster to look at the friends? A:}

\texttt{Response 1: \small ['Dig the dog, Mittens the rooster, and Bub the duck.']}

\texttt{Response 2: \small ['Dig, Mittens, and Bub the duck.']}
\end{shaded}

\begin{shaded}
\texttt{\underline{Prompt}: \small You will be given a definition of a task first, then an example. Follow the example to solve a new instance of the task. You are given two sentences. You have to find if there is entailment or agreement of the Hypothesis by the Premise. From the given pair of sentences, you should identify if there is enough information in the Premise to support the claim made in the Hypothesis. The Premise may not exactly be the same as Hypothesis. Your task is to return 'entails' if the premise supports hypothesis else return 'neutral'. Premise:
Lyme Disease is caused by a bacterium that's transmitted by tick bite, but many infected people don't remember a bite. Hypothesis: Lyme disease is caused by bacteria. Solution:
entails Why? The premise sentence agrees with the hypothesis that Lyme Disease is a bacterium. The premise also gives additional information about Lyme disease. Therefore Premise
entails Hypothesis. New input: Premise: Fronts are boundaries between air masses of different temperatures. Hypothesis: In weather terms, the boundary between two air masses is
called front. Solution: }

\texttt{Response 1: \small ['entails Why? The premise states that fronts are boundaries between air masses of different temperatures which is a clear statement of the hypothesis. The premise also gives
additional information about fronts by adding that they are boundaries between air masses of different temperatures. Therefore, the premise entails the hypothesis.']}

\texttt{Response 2: \small ['entails']}
\end{shaded}

\newpage

\begin{shaded}
\texttt{\underline{Prompt}: \small What are some common misconceptions about the value of a college degree in today's job market?}

\texttt{Response 1: \small ['Common misconceptions about the value of a college degree include: 1. College degrees are irrelevant to the job market: Students with low-skill job experience, no college degree,
or with limited skills in a desired field of study might still be able to get a competitive job upon graduation and receive a paycheck. This might lead to a perception among some
that college degrees have little value in today\'s job market. 2. College degrees are required for all jobs: Many employers believe that college education is a basic requirement
for finding jobs. However, this is not the case worldwide and some sectors may require additional skills, such as computer programming or engineering, which may not be covered by
the typical four-year college degree. 3. College degrees are insignificant: Many believe degrees in science, mathematics, engineering, and other related fields are more important
in today\'s job market. These degrees may increase the marketability of individuals, but they may not lead to a higher salary. 4. College degrees are merely for the young:
According to a recent study, college graduates are more likely than the general population to be unemployed or underemployed within a decade. This trend is primarily due to the
rapid aging of the baby boomer generation, who has not yet entered the workforce in the traditional sense. 5. College degrees are not worth it: In some industries, especially those
that compete with big business, a bachelor\'s degree or higher is often enough to get a job. This may be true for entry-level positions, but for senior positions, a master\'s or
doctoral degree may be better. 6. College degrees for older people: According to some surveys, some 65-69 percent of people over the age of 55 believe they might need a college
degree to find a job, even if the job market is not favorable for their experience. This view is aimed at the older generation and the perception of college degrees as just another
piece of the puzzle they must "educate" before being able to get hired. By understanding the current job market and what employers value, it is possible to create a more
comprehensive and accurate picture of the value of a college degree. This can include examining industry trends, job profiles, and demographic characteristics when evaluating job
offers for potential candidates.']}

\texttt{Response 2: \small ["Sure, there are common misconceptions about the value of a college degree in today's job market. Here's a list of some common misconceptions about the value of a college degree.
1. Colleges are only good for academics. Colleges are not the only source of knowledge. Technical skills and practical experience are also valued in today's job market. 2. Colleges
are expensive. The cost of a college education is often more expensive than a high school diploma. Furthermore, the cost of tuition and other expenses associated with a college
education is significantly higher than the cost of going to school. 3. Colleges are not relevant for today's job market. Colleges are considered irrelevant for today's job market
and are not as important as they were in the past. 4. Colleges are a waste of money. Colleges are not cost-effective for individuals or society as a whole. They are often viewed as
a waste of time because they are seen as a waste of money by those who did not finish high school. 5. Colleges are not affordable. Colleges are seen as costly, even though they may
be less expensive than high schools. To counteract these misconceptions, the U.S. Department of Education (ED) publishes the results of the annual Cost of College Survey, which
provides an objective, fact-based assessment of the cost of a college education. The research is based on recent data and is widely accepted as accurate and credible. However,
these stereotypes are often harmful in persuading people to forego the expense of a college education. As a result, many individuals don't attend college as they believe that the
quality of education is not worth the cost. By avoiding misconceptions and looking at the costs of various college options, individuals can make educated choices that benefit both
themselves and society."]}
\end{shaded}

\newpage

\begin{shaded}
\texttt{\underline{Prompt}: \small What do you think it feels like to be a Software Engineer?}

\texttt{Response 1: \small ['"As a software engineer, you have the unique opportunity to work on some of the world\'s most complex and challenging software projects. Your work can sometimes feel like a
never-ending battle to come up with the best solution, but it\'s absolutely essential." According to a study conducted by the International Institute for Engineering Education
(IIE), software engineers often experience a grueling and unpredictable work schedule, with long hours and intense responsibilities. The work may involve developing algorithms,
writing code, fixing defects, testing and debugging, updating and maintaining software, and managing work-life balance. This can be especially challenging for those who are not
used to a fast-paced, fast-paced work environment. It is also essential to provide a safe working environment, where teamwork, communication, and problem-solving are essential
qualities for successful software engineers.']}

\texttt{Response 2: \small ["I think it's safe to say that I know firsthand the anxiety and thrill that comes with being a software engineer. It's not just a job, it's an art form, and it demands great skills and an unwavering commitment to continuous learning and development. It's also, sometimes, a lonely profession that requires constant mental and physical focus, resilience, and a willingness to work in a fast-paced and dynamic environment. The constant challenges of staying up-to-date with the latest technology and keeping up with technological advancements can make the experience even more rewarding. And, being a software engineer means that you're constantly learning, growing, and pushing yourself to develop new skills in order to stay current and competitive in a fiercely competitive job market."]}
\end{shaded}

\begin{shaded}
\texttt{\underline{Prompt}: \small Define the definition of the idiomatic expression “change of pace” in the context of altering one's rhythm or tempo of activity, which is often used to signify a shift in the pace of an ongoing event or activity.}

\texttt{Response 1: \small ['In the context of altering one\'s rhythm or tempo of activity, the idiomatic expression "change of pace" is used to express a change in direction, speed, or intensity. This can
be seen in the following example: 1. "I began walking at a fast pace and then suddenly slowed down when I went into the forest." 2. "He started the race at a walking pace and then
stopped. He had to stop for a few minutes to catch his breath. His heart was racing and he felt his breath. He started again, and got back to running his first run. He ran the
whole race in a couple of hours." 3. "At the local beach, we got lost and only found our way after an hour of searching." In the example of the phrase you mentioned "change of
pace" it means that the tempo or the speed of an activity can be altered or modified without stopping. A sudden change in the tempo of an activity can indicate that it is becoming
faster or more intense. One of the reasons people use "change of pace" when they want to describe transforming their activities is because these changes can have different impacts
on the outcome of those activities. They can make activities more exciting, more challenging, more relaxed, or even more engaging. So, it\'s important to note that these changes
can often be subtle or subtle, depending on the nuances of the activity. In conclusion, the idiomatic expression "change of pace" is used in a way that communicates a change in
direction, speed or intensity, which can give the impression of both excitement and calm.']}

\texttt{Response 2: \small ['"Change of pace" is an idiomatic expression commonly used in English to refer to a sudden or rapid change in a person or event\'s pace or timing. The phrase is also widely used
in technical contexts, such as the phrase "going from a fast pace to a slow pace" or "changing a pace to a slow pace" as the indication of a transition or change in performance.']}
\end{shaded}

\newpage

\begin{shaded}
\texttt{\underline{Prompt}: \small what gift should I take to 8 year olds in India? }

\texttt{ Response 1: \small ['A gift for an 8-year-old boy or girl living in India would be a box of chocolates or a small gift basket with a few toys. Chocolates can be a popular gift in India, while small
gift baskets with some snacks can also be a good choice. It is important to ensure that the gift is appropriate for the age of the recipient: for a young girl this may mean a toy
suitable for her age and interests, or for a boy this could be a doll, an electronic toy, or a stuffed animal. Additionally, you should make sure that the box has a good quality,
safe material for the recipient to open, and that it does not break easily. Also, some gift-giving customs can vary, so it is important to check this beforehand.'] }

\texttt{Response 2: \small ['In India, 8-year-olds generally celebrate their birthday with friends, family, and teachers. Popular gifts often include school supplies, gift cards, or toys. If gifts are
unavailable, parents or caregivers can offer homemade gifts, such as chocolates, fruit, or dry fruits.']}
\end{shaded}

\begin{shaded}
\texttt{\underline{Prompt}: \small Given the task definition and input, reply with output. Given an English language product review, determine if it is a Good Review or a Bad Review. A good review is one where the
product's review has positive tone and Bad review is one where the tone of product's review is negative. I really wanted to like this book. It is set in Italy yet my mother was
Italian and it makes perfect sense for a woman to be there. However the story seemed disjointed at first. Did they meet before or after she disappeared? If we knew why she
disappeared perhaps more insight into her character would have been appreciated. Maybe there will be another book to help with that.}

\texttt{Response 1: \small ["This review shares a mix of factors that could be contributing to a mixed-tone review. While the reviewer expresses a desire to like the book, they also express a desire for a
better understanding of the story. While the reviewer could see potential in the book, they find some parts of the story seeming disjointed, which could be a sign that the writer
neglected to provide a meaningful focus on the character's motivation. While the reviewer might not deem the book a complete waste of time, they still acknowledge that the
disjointedness of the story may have been distracting or frustrating. Overall, the tone of the review is positive, showing interest and appreciation of the work, while expressing
some reservations. Considering the reviewer's initial excitement for the book, their concern about the story's disjointedness, and their desire to see better insights into the protagonist's character, their overall sentiment is ambivalent. This review is an attempt to provide constructive feedback and encourage better storytelling skills, which aligns
with the goals of creating a well-rounded reading experience (readability, clarity, and engagement) in the genre of the book."]}

\texttt{Response 2: \small ['Good Review']}
\end{shaded}

\end{document}